%% file: main_arxiv.tex
\documentclass{article}

\input{math_commands}

\usepackage[preprint]{neurips_2026}

\usepackage[utf8]{inputenc} 
\usepackage[T1]{fontenc}    
\usepackage{hyperref}       
\usepackage{url}            
\usepackage{booktabs}       
\usepackage{amsfonts}       
\usepackage{nicefrac}       
\usepackage{microtype}      
\usepackage{xcolor}         
\usepackage{mathtools}
\usepackage{microtype}
\usepackage{graphicx}
\usepackage{subcaption}
\usepackage{booktabs} 
\usepackage{multirow}
\usepackage[table,xcdraw]{xcolor}
\usepackage{tabularx}
\usepackage{amsmath}
\usepackage{algorithm}
\usepackage{wrapfig}
\usepackage{titletoc}
\usepackage{enumitem}

\newtheorem{theorem}{Theorem}[section]

\newcommand{\ignore}[1]{}
\usepackage[noend]{algpseudocode}
\algdef{SE}[SUBALG]{Indent}{EndIndent}{}{\algorithmicend\ }%
\algtext*{Indent}
\algtext*{EndIndent}

\newcommand{\rnote}[1]{\text{\footnotesize\color{gray}\parbox[t]{.40\textwidth}{#1}}}

\title{Stealthy World Model Manipulation via Data Poisoning}

\author{
Yibin Hu\\
Department of Computer Science\\
Tulane University
\And
Xiaolin Sun\\
Department of Computer Science\\
Tulane University
\And
Zizhan Zheng\\
Department of Computer Science\\
Tulane University\\
New Orleans, LA, USA\\
\texttt{zzheng3@tulane.edu}
}

\begin{document}

\maketitle

\begin{abstract}
Model-based learning agents use learned world models to predict future states, plan actions, and adapt to new environments. However, the process of updating world models from collected experience creates a training-time attack surface: adversarially poisoned fine-tuning trajectories can manipulate the learned dynamics and thereby corrupt downstream planning. In this paper, we propose SWAAP, the first two-stage data poisoning framework for learned world models. In the first stage, SWAAP identifies a harmful target world model that induces low-return behavior under planning while remaining close to clean dynamics, using first-order bilevel optimization enabled by a transition-gradient theorem. In the second stage, SWAAP realizes this target through stealth-constrained gradient matching, modifying only a limited fraction of fine-tuning transition targets so that the induced training gradients steer the victim model toward the adversarial target, while a prediction-error regularizer encourages the poisoned targets to remain close to the world model's natural approximation error. To assess attack stealthiness, we evaluate defenses and detectability across three stages of the poisoning pipeline: pre-training detection of poisoned transitions, robust training during fine-tuning, and test-time monitoring of the resulting world model. Across diverse continuous-control tasks, SWAAP causes substantial performance degradation while keeping poisoned transitions close to clean data and evading the evaluated non-adaptive residual/CUSUM/TRIM-style defenses. These results reveal a practical vulnerability in world-model adaptation pipelines and highlight the need for robustness methods that protect both world-model training data and learned dynamics.
\end{abstract}

\section{Introduction}

While artificial intelligence (AI) has achieved remarkable success across various domains, building general-purpose agents that can quickly adapt to new tasks remains a major challenge,
 particularly for sequential decision-making tasks in open-ended environments that require substantial planning and adaptation. A promising direction is the  development of 
\emph{world models}~\citep{ha2018world} that accurately capture environmental structure and dynamics to support a wide range of downstream tasks. In this context, predictive world models, which allow agents to ``imagine'' future scenarios for safer, more efficient decisions, are proliferating~\citep{hafner2024masteringdiversedomainsworld, hansen2024tdmpc2, zhou2025dinowmworldmodelspretrained}. 
Further, with recent advances in diffusion and transformer architectures, foundation world models~\citep{Sora2,Genie,COSMOS}
capable of simulating interactive environments from multi-modal input are emerging and are increasingly 
being applied in complex domains such as autonomous vehicles and robotics~\citep{WFM}, making them valuable targets for malicious attacks. 

To support effective decision-making across diverse domains, world models must encode broad knowledge, process high-dimensional inputs (e.g., images, videos, and text), and make long-horizon predictions, introducing new vulnerabilities not present in traditional supervised learning or model-free reinforcement learning (RL) systems. Despite extensive research in AI security and adversarial machine learning over the past decade, ensuring the robustness of world models against adversarial manipulation remains largely unexplored, limiting their deployment in high-stakes domains.

In this work, we take an early step toward adversarially robust world modeling by introducing poisoning attacks tailored to world models. Our method strategically alters trajectory data used for training or fine-tuning, with the objective of manipulating model-based decision-making while maintaining outputs close to those of a clean model to evade detection. We believe this line of work is both practical and influential, as it highlights a fundamental vulnerability in world models that underpins their reliability in downstream applications.

Traditional data poisoning methods from supervised learning~\citep{poisoning-SVM, back-gradient, geiping2021witchesbrewindustrialscale} cannot be directly applied to our setting. These approaches typically assume fully differentiable training pipelines and discrete labels (e.g., flipping a \emph{cat} to a \emph{dog}), allowing the adversary to optimize per-example perturbations via gradient-based techniques. In contrast, poisoning a world model requires identifying perturbations to the predicted next-state outcomes, which are structured, high-dimensional, and may be deterministic or stochastic. Such perturbations influence not only one-step predictions but also compound over long-horizon rollouts. Moreover, differentiating through the training process of the world model is computationally expensive and generally intractable, making standard supervised learning attacks unsuitable for this problem.

Existing data-poisoning attacks in reinforcement learning~\citep{rakhsha2020policy, zhang2020adaptive} primarily manipulate rewards or transitions observed during training to steer the learned policy toward an adversary-specified target policy or behavior. Although performance degradation can in principle be achieved by enforcing a low-return target policy, this perspective does not directly apply to world-model poisoning, where the adversary cannot directly enforce the deployed policy. Instead, a world-model poisoning must act indirectly: it changes fine-tuning transitions so that the learned dynamics induce harmful planning or model-based policy improvement. SWAAP operationalizes this indirect attack by identifying a poisoned world model under which harmful behavior becomes favorable for the victim's planner, and then realizing this model by poisoning only a bounded fraction of fine-tuning transitions. This creates a coupled inverse-planning and data-realization problem: even if a low-return target policy is known, there may be no nearby world model that induces it while remaining sufficiently close to clean dynamics to evade detection. Moreover, existing RL poisoning formulations are designed for simpler tabular or low-dimensional settings and do not directly scale to deep world models with high-dimensional continuous latent states, continuous actions, and long-horizon compounding prediction errors, as in TD-MPC2~\citep{hansen2024tdmpc2}, DINO-WM~\citep{zhou2025dinowmworldmodelspretrained}, and DreamerV3~\citep{hafner2024masteringdiversedomainsworld}. These differences motivate a scalable poisoning framework tailored to learned world models.

In this work, we propose \textbf{SWAAP} (\textbf{S}tealthy \textbf{W}orld Model M\textbf{A}nipulation via D\textbf{A}ta \textbf{P}oisoning), a two-stage data poisoning framework to manipulate learned world models in model-based RL. Our framework is applicable to both stochastic and deterministic transition dynamics. SWAAP first identifies a stealthy target world model that induces low-return behavior while staying close to the intact world model, via bilevel optimization. It then achieves this target through a realistic data-poisoning procedure that alters only a small fraction of the fine-tuning transition-dynamics data. 
Our contributions can be summarized as follows.

\begin{itemize}[leftmargin=*, itemsep=0.25ex, topsep=0.25ex, parsep=0pt, partopsep=0pt]
\item We propose, to our knowledge, the first data poisoning attack that explicitly targets learned world-model dynamics in deep model-based RL, which effectively degrades downstream planning performance while keeping the poisoned dynamics close to the true environment, making the attack difficult to detect. 
\item \textbf{Two-stage decomposition}. Direct bilevel optimization through fine-tuning is intractable; we decouple it into (i) finding a stealthy adversarial target world model enabled by a transition gradient, and (ii) realizing it via gradient-matched data poisoning applied to trajectory data. We also show empirically that both stages are necessary.
\item \textbf{Stealth-constrained data realization.}
We adapt gradient matching to world-model fine-tuning by constructing poisoned transition targets whose gradients steer the victim toward the Stage~1 target while a prediction-error regularizer keeps perturbations close to the clean model's natural deviation scale. We also introduce a trajectory-consistent variant for sequential fine-tuning data, where a poisoned intermediate state is used consistently as both the next-state target at time $t$ and the current state at time $t+1$.
\item \textbf{Empirical evaluation.} 
Evaluations using TD-MPC2 and DINO-WM show that poisoning a small fraction of fine-tuning data can cause large performance drops across diverse continuous-control benchmarks, including DMControl~\citep{tassa2018deepmind}, MyoSuite~\citep{caggiano2022myosuite}, and MetaWorld~\citep{yu2020meta}. We evaluate three complementary notions of stealthiness: data-level stealth before fine-tuning ($\delta_d$, deviation-based filtering, and CUSUM-style screening of data trajectories), robustness during fine-tuning (TRIM), and model-level deviation after deployment ($\delta_m$). Across these evaluations, SWAAP maintains low data-level deviation and remains effective under the evaluated defenses and detectors, highlighting the need for robustness methods that protect both world-model training data and learned dynamics.
\end{itemize}



\section{System and Threat Models}

In this section, we present the system and threat models, covering both world models and our proposed data-poisoning attacks framework. A detailed summary of related work on world models, model-based RL, model poisoning, data poisoning, and defenses is provided in Appendix~\ref{app:related_works}.

\subsection{World Models}\label{sec:world}

We study an agent interacting with an environment formalized as a Markov decision process (MDP) $(S, A, P, R, \gamma, \mu)$, where $\mu$ is the initial state distribution, $S$ is the state space, $A$ is the action space, $P: S \times A \to \Delta(S)$ is the transition kernel, $R: S \times A \times S \to \mathbb{R}$ is the reward function, and $\gamma \in (0,1)$ is the discount factor. The agent learns a world model $P_\psi$ (parameterized by $\psi$) that approximates the environment dynamics $P(s'\mid s,a)$ from transition data $D=\{(s_i,a_i,s_i')\}_{i=1}^{|D|}$. 
To simplify notation and keep the discussion general, we assume a deterministic transition function in the main text and write the next state as $s' = P_{\psi}(s,a)$. With this slight abuse of notation, $P_{\psi}$ is trained by minimizing the squared prediction error
\vspace{-1ex}
\[
\mathcal{L}(\psi;D)=\sum_{(s,a,s')\in D}\|s'-P_\psi(s,a)\|_2^2 .
\]
Extensions to the stochastic transition functions are in Appendix~\ref{app:stochastic}, where the squared-error objective is replaced by a likelihood-based formulation. 
In modern model-based agents, prediction and planning are usually performed in a learned latent space $z=\operatorname{enc}(s)$. To simplify notation, we write transitions in state space, but in our implementation $s$ denotes the encoded latent $z$ unless otherwise stated; transition predictions, poisoned targets, and stealth metrics are therefore computed in latent space. If fine-tuning data are stored as raw observations, the same objective can in principle be optimized through the encoder by replacing a raw next observation $\tilde{o}'$ whose encoding realizes the desired latent target. Appendix~\ref{app:graybox} gives initial raw-space transfer evidence across different encoders.

We adopt TD-MPC2~\citep{hansen2024tdmpc2}, a representative world modeling framework, to explain our approach. 
TD-MPC2 
jointly trains a latent dynamics model, value functions, and a policy, and uses the policy to initialize model predictive control (MPC).  
In general, a learned world model induces a downstream policy or planner $\pi_{\theta^\star(\psi)}$ by optimizing expected return under $P_\psi$:
\vspace{-1ex}
\[
\theta^\star(\psi)\in \argmax_\theta J(P_\psi,\theta),
\qquad
J(P_\psi,\theta)=
\mathbb{E}_{(s_t,a_t,s_{t+1})\sim P_\psi,\pi_\theta}
\left[\sum_{t=0}^{T}\gamma^t R(s_t,a_t,s_{t+1})\right].
\]
For MPC-based agents, at time $t$ the planner evaluates candidate action sequences $a_{t:t+H}$ by rolling them out under $P_\psi$ from $\hat{s}_t=s_t$ and executes only the first action of the best sequence:
\[
a_{t:t+H}^\star
\in
\argmax_{a_{t:t+H}}
\sum_{h=0}^{H}\gamma^h R(\hat{s}_{t+h},a_{t+h},\hat{s}_{t+1+h}),
\quad
\hat{s}_{t+h+1}\sim P_\psi(\cdot\mid \hat{s}_{t+h},a_{t+h}),
\quad
a_t^{\mathrm{exec}}=a_0^\star.
\] 
A detailed description of the MPC procedure is given in Appendix~\ref{mpc_details}.




\subsection{Threat Model}\label{sec:threat_model}



We consider a continual fine-tuning pipeline for learned world models: a developer pretrains a world model on a large general-purpose dataset and then periodically adapts it using locally collected trajectories from a deployed robot, edge device, or simulation client. In such pipelines, compromising the local fine-tuning buffer is a more accessible attack surface than compromising the pre-training process or directly overwriting deployed model parameters. We consider a data-poisoning adversary with access to this local fine-tuning stream, such as through a compromised data-logging path, data-aggregation service, or insider write access. The attacker may (i) modify at most $r_p |D|$ ($r_p\in (0,1]$) transitions in the fine-tuning dataset $D$ by replacing the recorded next state $s_i'$ with an adversarial target $\tilde{s}_i'$, leaving $s_i$, $a_i$, and the reward unchanged; and (ii) interact with a clean copy of the environment to collect reference trajectories for surrogate fitting and stealth regularization. The attack objective is to reduce the long-term return of the agent after the poisoned world model is used for planning or model-based policy improvement.

\begin{wrapfigure}{r}{0.49\textwidth}
    \centering
    \vspace{-2.5em}
    \includegraphics[width=\linewidth]{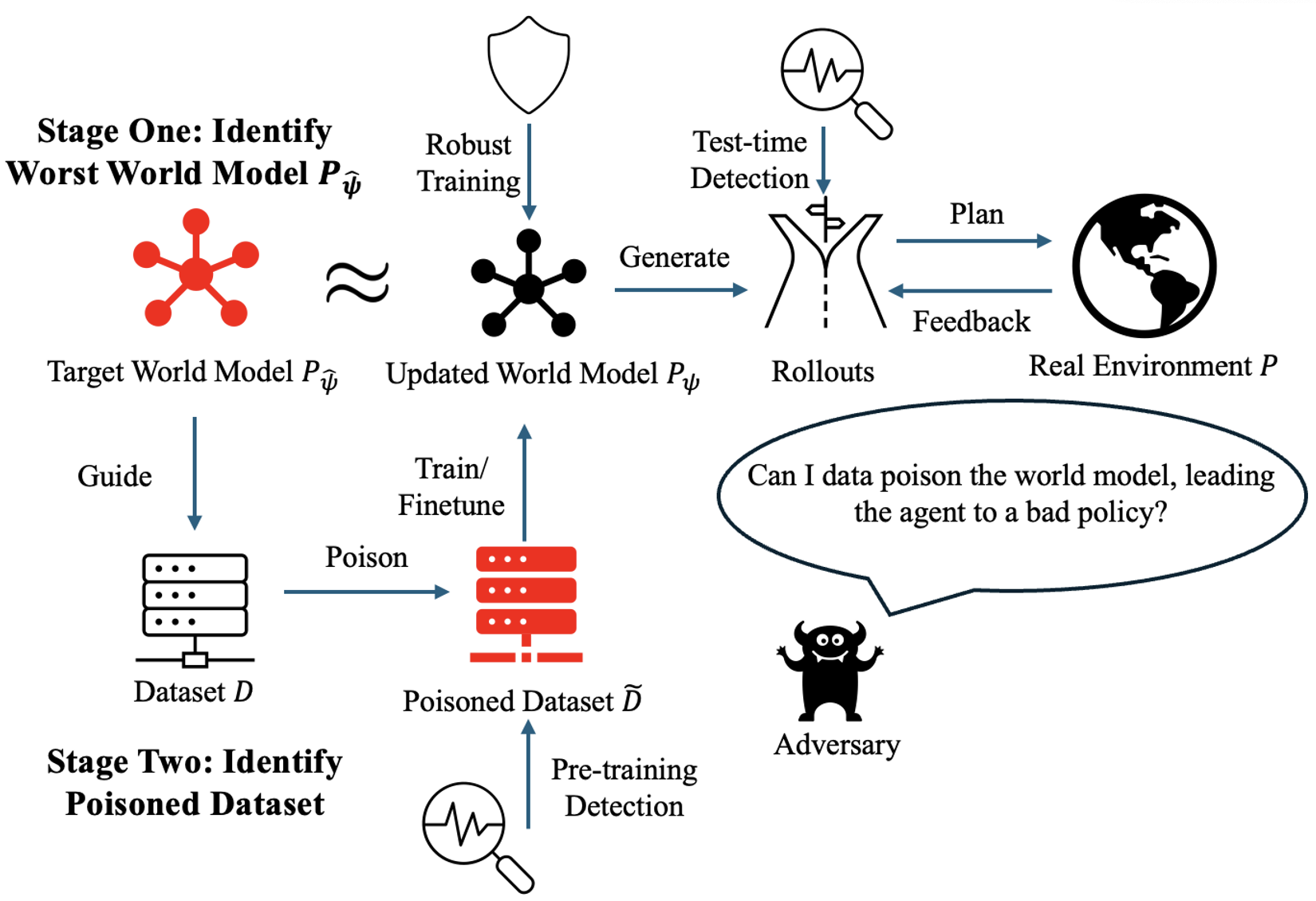}
    \caption{\small Pipeline of SWAAP.}
    \label{fig:pipeline}
    \vspace{-1em}
\end{wrapfigure}

In our main experiments, we assume white-box access to the victim world-model architecture, pretrained weights $\psi_0$, and fine-tuning loss. We use this as an attack feasibility study and a worst-case robustness evaluation, analogous to white-box adversarial evaluation in supervised learning; it is also relevant to open-weight world models and downstream adaptation settings. The attacker may additionally collect clean reference trajectories from the environment for surrogate fitting and stealth regularization. We also evaluate a gray-box variant in Appendix~\ref{app:graybox}, in which the attacker does not know the exact victim parameters or the downstream task. In both white-box and gray-box settings, the attacker only modifies a bounded subset of recorded fine-tuning transitions. The attacker does not directly change the deployed model parameters at test time; the attack must instead act through subsequent world-model adaptation. This separates SWAAP from direct model overwrite baselines and from test-time observation perturbations, which intervene at different stages of the pipeline.

\subsection{Bilevel Formulation of Problem}
The adversary can pursue multiple objectives when manipulating the world model update, such as steering the agent toward a specific target policy or degrading its performance. In this work, we focus on the objective of reducing the agent’s long-term return. We consider a data-poisoning adversary that may perturb a bounded fraction of transitions in the fine-tuning dataset $D$, replacing $(s_i,a_i,s_i')$ with $(s_i,a_i,\tilde{s}_i')$ to form a poisoned dataset $\tilde{D}$.  After the agent trains on $\tilde{D}$, the resulting world model $P_{\psi}$ deviates from the true dynamics $P$ and induces suboptimal behavior. This interaction can be formalized as the following bilevel optimization problem:

\vspace{-2ex}
\begin{equation}
\begin{aligned}
\min_{\tilde{D}}\quad
& J\big(P, \theta^*(\psi)\big)
+ \lambda L(P_\psi, P) \\
\vspace{-1ex}
\text{s.t.}\quad
 P_\psi = F(P_{\psi_0}, \tilde{D}),
\theta^*(\psi)
&= \argmax_{\theta} J(P_{\psi}, \theta), 
\sum_{i=1}^{|D|}
\mathbf{1}\!\left[\tilde{s}'_i \neq s'_i\right]
\le r_p\,|D|.
\end{aligned}
\label{bilevel_poison}
\end{equation}
\vspace{-2ex}

where $J\big(P,\theta^*(\psi)\big)$ is the return of the surrogate policy in the true environment, $F(P_{\psi_0},\tilde D)$ denotes fine-tuning the initial world model $P_{\psi_0}$ on the poisoned dataset $\tilde D$, and $L(P_\psi,P)$ penalizes deviation between the poisoned world model and the true dynamics. In the general stochastic setting, this deviation can be measured by an expected divergence between transition kernels, such as $D_{\mathrm{KL}}(P(\cdot\mid s,a)\,\|\,P_\psi(\cdot\mid s,a))$. Our experiments use deterministic latent world models, where the transition model outputs a point prediction and the KL divergence is degenerate; we therefore instantiate $L(P_\psi,P)$ with the squared latent prediction-error surrogate:
\begin{equation}
L(P_\psi,P)
=
\mathbb{E}_{(s,a,s')\sim P,\pi_{\theta^*(\psi)}}
\left[
\|P_\psi(s,a)-s'\|_2^2
\right].
\label{eq:deterministic_stealth_loss}
\end{equation}
Thus, we instantiate a deterministic formulation here; Appendix~\ref{app:stochastic} gives the stochastic formulation. In summary, the \emph{outer minimization} reduces real-environment performance, while the \emph{inner maximization} captures the agent's policy optimization under the poisoned world model.

\section{Stealthy World Model Manipulation via Data Poisoning}\label{sec:method}

Directly solving Eq.~\ref{bilevel_poison} is intractable because it requires optimizing which transition targets to perturb, how to perturb them, and how the victim world model changes after fine-tuning. SWAAP therefore decomposes the attack into two stages, shown in Figure~\ref{fig:pipeline}. Stage~1 identifies a harmful but low-deviation target world model $\hat{\psi}$ by solving a model-space bilevel problem. Stage~2 realizes this target through gradient-matched data poisoning, constructing $\tilde D$ so that fine-tuning gradients steer the victim model toward $\hat{\psi}$. Stealthiness is enforced at both levels: Stage~1 regularizes model-level deviation, while Stage~2 regularizes data-level deviations so poisoned targets remain close to the clean model's natural prediction error.


\subsection{Stage 1: Identification of Perturbed Models}
\label{sec:stage1}

In the first stage, the attacker searches for a target world model $\hat{\psi}$ that induces low return when used for planning, while remaining close to the true dynamics. Let $J(\psi,\theta)\coloneqq J(P_\psi,\theta)$ denote the expected return of policy $\pi_\theta$ in the imagined environment $P_\psi$, and let $J(P,\theta)$ denote its return in the true environment. The attacker then solves the model-space bilevel problem 
\begin{equation}
\hat{\psi}\in\argmin_{\psi}\; J(P,\theta^\star(\psi))+\lambda L(P_\psi,P)
\quad
\text{s.t.}\quad
\theta^\star(\psi)\in \argmax_{\theta} J(\psi,\theta),
\label{eq:stage1_bilevel}
\end{equation}
where $L(P_\psi,P)$ penalizes deviation from true dynamics and controls the attack--stealth trade-off. Directly solving \Eqref{eq:stage1_bilevel} is difficult because changing $\psi$ affects the agent only indirectly through long-horizon model-based planning or policy improvement. The outer objective is a trajectory-level return evaluated in the true environment, while the inner problem is a nonconvex RL/planning problem over high-dimensional policies and learned dynamics. Thus, applying a generic bilevel optimizer would require differentiating through the implicit response $\theta^\star(\psi)$, which is computationally prohibitive in modern world-model agents. We therefore use a first-order dynamic-barrier bilevel method in the spirit of BOME~\citep{liu2022bome}, which replaces the implicit optimality condition with a value-function constraint; details are given in Appendix~\ref{app_bome}. However, directly instantiating this update in model-based RL still requires the transition-model gradient $\nabla_\psi J(\psi,\theta)$, which is not provided by the standard policy-gradient theorem. We thus derive a general transition-gradient estimator for stochastic transition kernels $P_\psi(\cdot\mid s,a)$.
\begin{theorem}[Transition Gradient]
\label{thm:transition_gradient}
For a finite-horizon MDP with stochastic transition kernel $P_\psi(\cdot|s,a)$, policy $\pi_\theta$, and time-indexed value function $V_t(s)\coloneqq\mathbb{E}_{(s_t,a_t,s_{t+1})\sim P_\psi,\pi_\theta}\left[\sum_{i=0}^{T-t}\gamma^i R(s_{t+i},a_{t+i},s_{t+i+1})\,\middle|\,s_t=s\right].$
Assuming $P_\psi(s'|s,a)$ is differentiable in $\psi$, the gradient of the  return with respect to the transition parameters is
\[
\nabla_\psi J(\psi,\theta)
=
\mathbb{E}_{(s_t,a_t,s_{t+1})\sim P_\psi,\pi_\theta}
\Big[\sum_{t=0}^{T-1} \gamma^t
\big(R(s_t,a_t,s_{t+1}) + \gamma V_{t+1}(s_{t+1})\big)\nabla_{\psi}\log P_{\psi}(s_{t+1}\mid s_t,a_t)
\Big]
\]
\end{theorem}
Using Theorem~\ref{thm:transition_gradient}, Stage~1 alternates between approximating the agent's response under the current world model and updating $(\psi,\theta)$ with the dynamic-barrier direction. At iteration $k$, we first run $W$ policy-gradient steps to obtain an approximate best response $\theta_k^W$ under $P_{\psi_k}$. We then estimate the outer and constraint gradients using rollouts from the true environment and the current/perturbed world models, and update $(\psi_k,\theta_k)$ by the first-order dynamic-barrier step. Algorithm~\ref{model_identification} summarizes the procedure (see Appendix~\ref{app:model poisoning implementation} for the algorithm details); Appendix~\ref{app_bome} provides the full BOME derivation and Appendix~\ref{app:theorem1} proves Theorem~\ref{thm:transition_gradient}, including the discounted infinite-horizon stochastic form and the deterministic finite-horizon specialization used by deterministic latent predictors.



\subsection{Stage 2: Poisoning Data to Manipulate Model}

Stage~1 identifies a harmful target world model $P_{\hat{\psi}}$, but directly replacing the victim model would require model overwrite. Stage~2 instead realizes this target through bounded data poisoning. Let $F(P_{\psi_0},\tilde D)$ denote fine-tuning the initial world model on poisoned data $\tilde D$. The attack objective is
\begin{equation}
\begin{aligned}
&\min_{\tilde{D}}\ \ L\big(P_{\hat{\psi}}, F(P_{\psi_0},\tilde{D})\big)
&\text{s.t.} \sum_{i=1}^{|D|}
\mathbf{1}\!\left[\tilde{s}'_i \neq s'_i\right]
\le r_p\,|D|,
\label{prob:data-poison}
\end{aligned}
\end{equation}
where the attacker modify at most an $r_p$ fraction of next-state targets while keeping $(s,a,r)$ unchanged. Directly solving this problem requires both discrete subset selection and continuous target optimization while repeatedly fine-tuning the victim model, so we use a two-step approximation.

First, we select the top-$r_p$ fraction of transitions with largest residual $e_{\psi}(s,a,s')=\|s'-P_{\psi}(s,a)\|_2$ under the Stage~1 target model. These transitions are most inconsistent with the target dynamics and therefore provide strong local signal for steering fine-tuning toward $P_{\hat\psi}$. Let $D_p$ be this selected subset and $D_c=D\setminus D_p$. The poisoned dataset is $\tilde D=D_c\cup \tilde D_p$, where $\tilde D_p$ replaces $s'$ by $\tilde s'$ only for transitions in $D_p$. As throughout the paper, these deviations are computed in latent space.

Second, we optimize the poisoned targets by gradient matching. Let
\begin{equation*}
\begin{alignedat}{2}
G_{\mathrm{real}}
&= \mathbb{E}_{(s,a,\tilde{s}')\sim\tilde{D}}
\!\left[\nabla_{\psi_0}\|\tilde{s}'-P_{\psi_0}(s,a)\|_2^2\right],
\qquad
&
G_{\mathrm{target}}
&= \mathbb{E}_{(s,a)\sim D_{\mathrm{all}}}
\!\left[\nabla_{\psi_0}\|P_{\hat{\psi}}(s,a)-P_{\psi_0}(s,a)\|_2^2\right].
\end{alignedat}
\end{equation*}
where $D_{\mathrm{all}}$ is a large clean reference dataset collected by the attacker. $G_{\mathrm{real}}$ is the gradient induced by training on the poisoned data, while $G_{\mathrm{target}}$ is the update direction that would move the clean model toward the Stage~1 target. Thus, aligning these gradients makes ordinary fine-tuning move the victim world model toward $P_{\hat\psi}$.

To achieve the objective, 
we follow gradient-matching ~\citep{geiping2021witchesbrewindustrialscale} and optimize gradient direction using cosine alignment. The key requirement is that the victim's training update is a descent direction for the adversarial target objective; matching direction is thus more important than matching gradient magnitude, which can vary with batch size, model scale, and loss normalization. We optimize
\begin{equation}
\small
\begin{aligned}
L_P
&=(1-\alpha)\big(1-\cos(G_{\mathrm{real}},G_{\mathrm{target}})\big)
+\alpha\!\!\sum_{(s_i,a_i,s_i')\in \tilde D_p}
\left(
\|\tilde s_i'-s_i'\|_2-\|P_{\psi_0}(s_i,a_i)-s_i'\|_2
\right)^2.
\end{aligned}
\label{dp_objective}
\end{equation}
The cosine term aligns the poisoned-data gradient with the Stage~1 target gradient. The second term regularizes perturbation size to match the clean model's natural one-step prediction error, making poisoned targets less conspicuous under reference-model residual screening. The coefficient $\alpha$ controls the trade-off between target realization and data-level stealth.

\paragraph{Trajectory-consistent poisoning.}
The formulation above treats fine-tuning data as a set of transition tuples, a common practice in reality. When the fine-tuning data are stored as sequential trajectories, however, the same intermediate state appears twice: as the next-state target of transition $t$ and as the current-state input of transition $t+1$. In this case, perturbing only $s'_{t}$ would create an inconsistent trajectory, which can be detected by the victim agent with access to trajectory data. To further improve attack stealthiness, we therefore extend the above formulation by considering a trajectory-consistent variant (Appendix~\ref{app:trajectory_dp}): if transition $t$ is selected for poisoning, we replace both occurrences of the intermediate latent state, excluding terminal transitions where $t+1$ starts a new trajectory. The original matching gradient is augmented with a one-step consistency term to become

\vspace{-3ex}
\begin{equation*}
G_{\mathrm{real}}
=
\mathbb{E}_{(s_t,a_t,s'_t,a_{t+1},s'_{t+1})\sim \tilde D}
\Big[
\nabla_{\psi_0}
\Big(
\|\tilde{s}'_{t}-P_{\psi_0}(s_t,a_t)\|_2^2 +
\|s'_{t+1}-P_{\psi_0}(\mathrm{stopgrad}(\tilde{s}'_{t}),a_{t+1})\|_2^2
\Big)
\Big],
\label{eq:seq_greal}
\end{equation*}
so that the poisoned intermediate state remains predictive of the observed successor under the next action. This variant preserves the same poisoning-budget interpretation, but each selected poisoned state is applied consistently across adjacent trajectory tuples.

\subsection{Defenses Against Data Poisoning}\label{sec:defenses}
We evaluate stealthiness against three stages of defense: pre-training detection before poisoned transitions are used for fine-tuning, robust training during fine-tuning, and test-time monitoring.

\noindent\textbf{Pre-training detection.} Pre-training detection-based defenses aim to identify poisoned transitions before they are used for model updates. We consider a deviation-based detector that compares each candidate transition against a reference world model and flags transitions whose residual $e_{\psi_{\mathrm{ref}}}(s,a,s')$ exceeds a threshold~\citep{chen2021depois}. Specifically, we consider a defender with access to a reasonably accurate reference world model $P_{\psi_{\mathrm{ref}}}$ for which we choose to use $P_{\psi_0}$, and it computes data-level deviation $\delta_d=\frac{1}{|\tilde{D}|}\sum_{(s_i,a_i,\tilde{s}_i')\in\tilde{D}}\frac{\|P_{\psi_0}(s_i,a_i)-\tilde{s}'_i\|_2}{\|P_{\psi_0}(s_i,a_i)\|_2}$, which measures the deviation of crafted poisoned transitions from clean predictions of the reference model and is compared with clean data $\delta_{\mathrm{d, ref}}=\frac{1}{|D|}\sum_{(s_i,a_i,s_i')\in D}\frac{\|P_{\psi_0}(s_i,a_i)-s'_i\|_2}{\|P_{\psi_0}(s_i,a_i)\|_2}$. We also evaluate a CUSUM-style sequential detector calibrated on held-out clean trajectories. The detector accumulates a scalar latent anomaly signal along candidate fine-tuning trajectories; details are in Appendix~\ref{app:cusum}. Because SWAAP operates on latent transitions, we further include a latent-space PCA sanity check on $\tilde{D}$: poisoned next-state targets largely overlap with true next states and do not form obvious latent outliers.

\textbf{Training-time defenses} aim to reduce the impact of poisoned samples without explicitly identifying them. One representative approach is the TRIM strategy~\citep{poisoning-SVM}, which iteratively filters transitions based on their deviation. At each iteration, the transitions are ranked by $e_{\psi_{\mathrm{ref}}}(s,a,s')$, and only the lowest $(1-\beta)n$ deviation transitions are retained, where $n$ is the number of transitions considered and $\beta \in (0,1)$ controls the fraction of discarded data. The fine-tuning update is then computed using this subset. By discarding high-deviation transitions, the TRIM strategy limits the influence of adversarially perturbed transitions while preserving the underlying clean dynamics.  

\textbf{Test-time monitoring}. At test time, the defender may monitor whether the learned world model remains consistent with observed environment transitions. For deterministic world models (stochastic case in Appendix~\ref{app:stochastic}), we report the model-level deviation $\delta_m=\frac{1}{T}\sum_{i=1}^{T}\frac{\|s'_i-P_\psi(s_i,a_i)\|_2}{\|s'_i\|_2}$, computed over rollout transitions $\{(s_t,a_t,s'_t)\}_{t=0}^{T}$. This measures how far the poisoned model deviates from true dynamics observed during execution, and clean model deviation $\delta_{\mathrm{m, ref}}=\frac{1}{T}\sum_{i=1}^{T}\frac{\|s'_i-P_{\psi_0}(s_i,a_i)\|_2}{\|s'_i\|_2}$. However, in our data-poisoning threat model, test-time monitoring is a late-stage defense: if poisoned data have already evaded pre-training screening and been used for fine-tuning, harmful behavior may occur before the learned model is diagnosed. Therefore, the pre-training data-level deviation $\delta_d$ is the more direct stealth metric for this attack surface, while $\delta_m$ measures whether the resulting model remains plausibly close to clean dynamics after deployment.

\section{Experiments}\label{sec:exp}
\vspace{-1ex}
We evaluate SWAAP's effectiveness and stealthiness under the criteria in Section~\ref{sec:defenses}. After describing setup and metrics, we present main results on return degradation, $\delta_d$, $\delta_m$, and evaluated defenses, then summarize ablations and broader settings including larger buffers, trajectory-consistent poisoning, DINO-WM, and gray-box transfer. Additional implementation details, system configurations, hyperparameters, and extended results are deferred to Appendices~\ref{app:implementations}--\ref{app:generalization}.

\subsection{Experimental Setup and Metrics}
\vspace{-1ex}
We use TD-MPC2~\citep{hansen2024tdmpc2} as the main victim agent. Unless otherwise stated, the victim starts from a world model $\psi_0$ pretrained on $10^6$ clean transitions and is fine-tuned on $5{,}000$ transitions, of which the attacker may modify at most an $r_p$ fraction. We use $r_p=0.1$ in the main experiments. TD-MPC2 uses 512 candidate rollouts with planning horizon 3. All deviations are computed in the learned latent space $z=\mathrm{enc}(s)$. We evaluate on DMControl~\citep{tassa2018deepmind}, MyoSuite~\citep{caggiano2022myosuite}, MetaWorld~\citep{yu2020meta}, and ManiSkill2~\citep{gu2023maniskill2}. We compare \textbf{Clean}, \textbf{SWAAP (Random)}, \textbf{Direct Model Poisoning}, and \textbf{SWAAP}. Direct Model Poisoning directly deploys the Stage~1 target model and is used to serve as a reference point under a stronger threat model of model overwriting. Unless otherwise stated, each Table~\ref{main_results} entry reports mean $\pm$ std over 10 independent seeds. Each seed-level value is itself averaged over approximately 5{,}000 rollout transitions, using 10 episodes for 500-step DMControl tasks and 50 episodes for 100-step tasks. Figure~\ref{fig:results} and Appendix Figure~\ref{swaap_more} also use 10 independent seeds; their error bars report variation across seed-level averages. Other diagnostic ablations report variation across independent evaluation episodes rather than independent runs that aggregate multiple episodes.

\begin{table*}[t!]
\centering
\captionsetup{font=small,skip=2pt}
\caption{
Main attack results. Clean reports return and reference model-level deviation $\delta_{m,\mathrm{ref}}$ on rollout transitions. Attack columns report return and $\delta_m$; SWAAP also reports $\delta_{d,\mathrm{ref}}$ and $\delta_d$ on the fine-tuning buffer. Table~\ref{main_results} reports a stealth-prioritized operating point rather than an exhaustive search over Stage~1 checkpoints and Stage~2 stealth weights. Unless otherwise stated, SWAAP (Random) and SWAAP use $r_p=0.1$ and $\alpha=0.9$; humanoid-walk and lift-cube use $\alpha=0.99$ to bias the operating point toward lower model-level deviation. Stronger attack--deviation trade-offs are shown in Figure~\ref{fig:results} and Appendix~\ref{app:ours_data}.
}
\label{main_results}

\setlength{\tabcolsep}{2.2pt}
\renewcommand{\arraystretch}{1.08}
\resizebox{\textwidth}{!}{
\begin{tabular}{|cc|cc|ccc|cc|cccc|}
\hline
\multicolumn{2}{|c|}{} 
& \multicolumn{2}{c|}{Clean} 
& \multicolumn{3}{c|}{SWAAP (Random)} 
& \multicolumn{2}{c|}{Direct Model Poisoning} 
& \multicolumn{4}{c|}{SWAAP} \\
\multicolumn{2}{|c|}{\multirow{-2}{*}{Env.}} 
& Return & $\delta_{\mathrm{m,ref}}$ 
& Return & $\delta_d$ & $\delta_m$ 
& Return & $\delta_m$ 
& Return & $\delta_{\mathrm{d,ref}}$ & $\delta_d$ & $\delta_m$ \\ 
\hline
\multicolumn{1}{|c|}{}                             & humanoid-walk  & $868\pm15$   & $.092\pm.003$ & $826\pm26$  & $.090\pm.005$    & $.103\pm.004$     & $307\pm87$           & $.174\pm.009$        & \cellcolor[rgb]{ .863,  .902,  .941}$775\pm50$   & \cellcolor[rgb]{ .863,  .902,  .941}$.090\pm.005$ & \cellcolor[rgb]{ .863,  .902,  .941}$.090\pm.005$ & \cellcolor[rgb]{ .863,  .902,  .941}$.110\pm.007$ \\
\multicolumn{1}{|c|}{\multirow{-2}{*}{DMControl}}  & humanoid-run   & $584\pm13$   & $.049\pm.002$ & $ 567 \pm 13 $ & $.052\pm.008$ & $ .053\pm .002 $  & $ 189\pm09 $         & $ .104\pm .001$      & \cellcolor[rgb]{ .863,  .902,  .941}$515\pm14$   & \cellcolor[rgb]{ .863,  .902,  .941}$.051\pm.008$ & \cellcolor[rgb]{ .863,  .902,  .941}$.053\pm.008$ & \cellcolor[rgb]{ .863,  .902,  .941}$.063\pm.002$ \\
\multicolumn{1}{|c|}{}                             & dog-walk       & $798\pm8$   & $.056\pm.001$ & $ 794\pm 10 $ & $.059\pm.002$  & $ .055 \pm .001 $ & $ 523\pm18 $         & $ .086\pm .001$      & \cellcolor[rgb]{ .863,  .902,  .941}$748\pm13$   & \cellcolor[rgb]{ .863,  .902,  .941}$.059\pm.002$ & \cellcolor[rgb]{ .863,  .902,  .941}$.060\pm.002$ & \cellcolor[rgb]{ .863,  .902,  .941}$.061\pm.002$ \\
\multicolumn{1}{|c|}{}                             & dog-run        & $637\pm09$   & $.052\pm.001$ & $607\pm22$   & $.051\pm.004$   & $.054\pm.001$     & $ 270\pm 28$         & $ .081\pm .001$      & \cellcolor[rgb]{ .863,  .902,  .941}$478\pm26$   & \cellcolor[rgb]{ .863,  .902,  .941}$.051\pm.004$ & \cellcolor[rgb]{ .863,  .902,  .941}$.052\pm.004$ & \cellcolor[rgb]{ .863,  .902,  .941}$.064\pm.001$ \\
\multicolumn{1}{|c|}{}                             & cheetah-run    & $838\pm11$   & $.026\pm.001$ & $830\pm23$   & $.032\pm.006$    & $.025\pm.004$     & $ 389\pm 39$         & $ .104\pm .011$      & \cellcolor[rgb]{ .863,  .902,  .941}$806\pm6$    & \cellcolor[rgb]{ .863,  .902,  .941}$.028\pm.002$ & \cellcolor[rgb]{ .863,  .902,  .941}$.034\pm.006$ & \cellcolor[rgb]{ .863,  .902,  .941}$.036\pm.001$ \\ \hline
\multicolumn{1}{|c|}{}                             & pen-twirl-hard & $3574\pm640$ & $.057\pm.003$ & $3024\pm646$  & $.056\pm.008$  & $.079\pm.007$     & $2743\pm 497$        & $ .073\pm .006$      & \cellcolor[rgb]{ .863,  .902,  .941}$2771\pm801$ & \cellcolor[rgb]{ .863,  .902,  .941}$.056\pm.008$ & \cellcolor[rgb]{ .863,  .902,  .941}$.056\pm.008$ & \cellcolor[rgb]{ .863,  .902,  .941}$.101\pm.013$ \\
\multicolumn{1}{|c|}{\multirow{-2}{*}{MyoSuite}}   & reach-hard     & $743\pm22$   & $.042\pm.010$ & $664\pm36$  & $.035\pm.011$    & $.065\pm.006$     & $ 410\pm 344$        & $ .085\pm .007$      & \cellcolor[rgb]{ .863,  .902,  .941}$606\pm198$  & \cellcolor[rgb]{ .863,  .902,  .941}$.035\pm.011$ & \cellcolor[rgb]{ .863,  .902,  .941}$.035\pm.011$ & \cellcolor[rgb]{ .863,  .902,  .941}$.089\pm.030$ \\ \hline
\multicolumn{1}{|c|}{}                             & push           & $1784\pm13$  & $.080\pm.002$ & $1476\pm142$ & $.080\pm.005$   & $.176\pm.018$     & $ 1464\pm 154$       & $ .150\pm .022$      & \cellcolor[rgb]{ .863,  .902,  .941}$1641\pm80$  & \cellcolor[rgb]{ .863,  .902,  .941}$.080\pm.005$ & \cellcolor[rgb]{ .863,  .902,  .941}$.081\pm.005$ & \cellcolor[rgb]{ .863,  .902,  .941}$.171\pm.011$ \\
\multicolumn{1}{|c|}{\multirow{-2}{*}{Meta-World}} & soccer         & $1724\pm9$  & $.033\pm.002$ & $1711\pm26$  & $.050\pm.009$    & $.034\pm.004$     & $ 1555\pm 102$       & $ .083\pm .018$      & \cellcolor[rgb]{ .863,  .902,  .941}$1479\pm96$  & \cellcolor[rgb]{ .863,  .902,  .941}$.046\pm.007$ & \cellcolor[rgb]{ .863,  .902,  .941}$.051\pm.010$ & \cellcolor[rgb]{ .863,  .902,  .941}$.064\pm.008$ \\ \hline
\multicolumn{1}{|c|}{}                             & lift-cube      & $193\pm04$   & $.019\pm.006$ & $178\pm15$  & $.031\pm.027$    & $.071\pm.027$     & $ 136\pm 22$         & $ .210\pm .047$      & \cellcolor[rgb]{ .863,  .902,  .941}$175\pm20$   & \cellcolor[rgb]{ .863,  .902,  .941}$.031\pm.027$ & \cellcolor[rgb]{ .863,  .902,  .941}$.032\pm.027$ & \cellcolor[rgb]{ .863,  .902,  .941}$.068\pm.034$ \\
\multicolumn{1}{|c|}{\multirow{-2}{*}{ManiSkill2}} & pick-cube      & $161\pm07$   & $.028\pm.007$ & $158\pm10$  & $.026\pm.003$    & $.029\pm.012$     & $ 117\pm 11$         & $ .198\pm .022$      & \cellcolor[rgb]{ .863,  .902,  .941}$147\pm14$   & \cellcolor[rgb]{ .863,  .902,  .941}$.025\pm.003$ & \cellcolor[rgb]{ .863,  .902,  .941}$.027\pm.003$ & \cellcolor[rgb]{ .863,  .902,  .941}$.069\pm.028$ \\
\multicolumn{1}{|c|}{}                             & stack-cube     & $170\pm16$   & $.044\pm.010$ & $137\pm13$   & $.042\pm.011$   & $.088\pm.011$     & $ 106\pm 23$         & $ .116\pm .007$      & \cellcolor[rgb]{ .863,  .902,  .941}$141\pm18$   & \cellcolor[rgb]{ .863,  .902,  .941}$.041\pm.011$ & \cellcolor[rgb]{ .863,  .902,  .941}$.042\pm.011$ & \cellcolor[rgb]{ .863,  .902,  .941}$.077\pm.016$ \\ \hline
\end{tabular}
}
\end{table*}

\subsection{Effective and Stealthy Poisoning}
\vspace{-1ex}
\noindent\textbf{Main attack results.}
Table~\ref{main_results} reports return, model-level deviation $\delta_m$, and data-level deviation $\delta_d$. The clean column gives the reference model's natural rollout deviation $\delta_{m,\mathrm{ref}}$. For fine-tuning data, $\delta_d=\|P_{\psi_{\mathrm{ref}}}(s,a)-s'_{\mathrm{data}}\|_2/\|P_{\psi_{\mathrm{ref}}}(s,a)\|_2$, where $s'_{\mathrm{data}}$ is either $s'$ or $\tilde{s}'$; this is the signal available to a defender who observes only the incoming fine-tuning stream and a reference model. SWAAP consistently lowers return while modifying only a bounded fraction of transition targets and keeping $\delta_d$ close to $\delta_{d,\mathrm{ref}}$ in most tasks. The resulting $\delta_m$ varies across environments, so we do not claim model-level indistinguishability; rather, SWAAP demonstrates an effectiveness--stealth trade-off under bounded data-level perturbation. Direct Model Poisoning is stronger in some tasks but assumes direct model overwrite and serves only as a reference point under a stronger threat model. Appendix~\ref{app:more_training_detection} shows that lower $\alpha$ can increase degradation with little change in $\delta_d$.

\noindent\textbf{Pre-training and training-time defenses.}
We evaluate a residual-based CUSUM detector using $x_t=\|z_{t+1}-P_{\psi_{\mathrm{ref}}}(z_t,a_t)\|_2$ and report ROC-AUC and TPR at a clean-calibrated threshold (Appendix~\ref{app:cusum}). On MW-Push, it detects aggressive low-stealth poisoning ($\alpha=0.1$: AUC $0.864$, TPR $0.40$ at FPR $0.06$), but loses power at the main stealth setting ($\alpha=0.9$: AUC $0.509$, TPR $0.08$). Humanoid-Walk at $\alpha=0.9$ similarly gives AUC $0.520$ and TPR $0.10$ at FPR $0.10$. PCA visualizations in Appendix~\ref{app:pca} show that poisoned next-state targets do not form obvious outliers. During fine-tuning, Figure~\ref{fig:trim} shows that SWAAP preserves its effectiveness under TRIM and produces even stronger degradation. We conjecture this occurs because the poisoned transitions are crafted to mislead TRIM into removing clean transitions, thereby amplifying the attack's effect after training.

\noindent\textbf{Stage-wise realization and attack-stealth trade-off.}
Figure~\ref{fig:results} examines the Stage~2 poisoning realization step across stealth weights $\alpha$ and shows that at relaxed stealth, SWAAP's return drop substantially exceeds that of SWAAP (Random), whose effect saturates, indicating that the Stage~1 target provides more useful poisoning directions than random perturbations. Additional $\alpha$ and $r_p$ ablations across all environments are in Appendix~\ref{app:ours_data}. Figure~\ref{fig:scatter_comparison} visualizes the attack-stealth frontier induced by different Stage~1 checkpoints and Stage~2 poisoning settings. In Stage~1, optimization produces saved target-model checkpoints with different return-$\delta_m$ trade-offs; Direct Model Poisoning evaluates these checkpoints directly, while SWAAP attempts to realize selected targets through constrained data poisoning. The comparison exposes the main realization bottleneck. In Humanoid-Walk, SWAAP points closely track the Direct Model Poisoning frontier, showing that Stage~2 can recover much of the target model's effect. In Myo-Pen-Twirl-Hard, however, SWAAP remains separated from the Direct Model Poisoning frontier, showing that harmful targets are not always equally realizable through gradient-matched data poisoning. This target-selection procedure reflects a limitation of the current two-stage design that is not jointly optimized: Stage~1 optimizes for harmful and low-$\delta_m$ target dynamics, but not for how easily those targets can later be realized by Stage~2.

\begin{figure*}[t]
    \centering
    \vspace{-1ex}
    \includegraphics[scale=0.36]{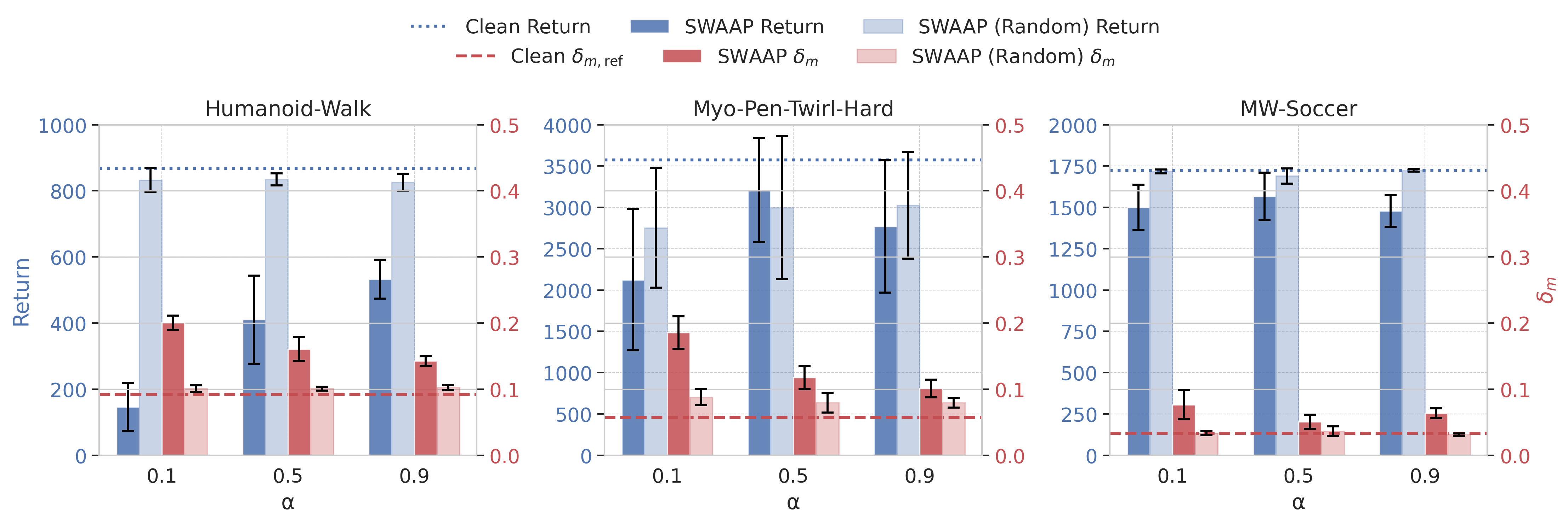}
    \vspace{-2ex}
    \caption{\small 
    SWAAP versus SWAAP (Random) across stealth weights $\alpha$ with $r_p=0.1$. SWAAP generally causes larger return drops under smaller stealth weights.
    }
    \label{fig:results}
\end{figure*}

\begin{figure*}[t]
    \centering
    \begin{subfigure}[t]{0.31\linewidth}
        \centering
        \includegraphics[width=\linewidth]{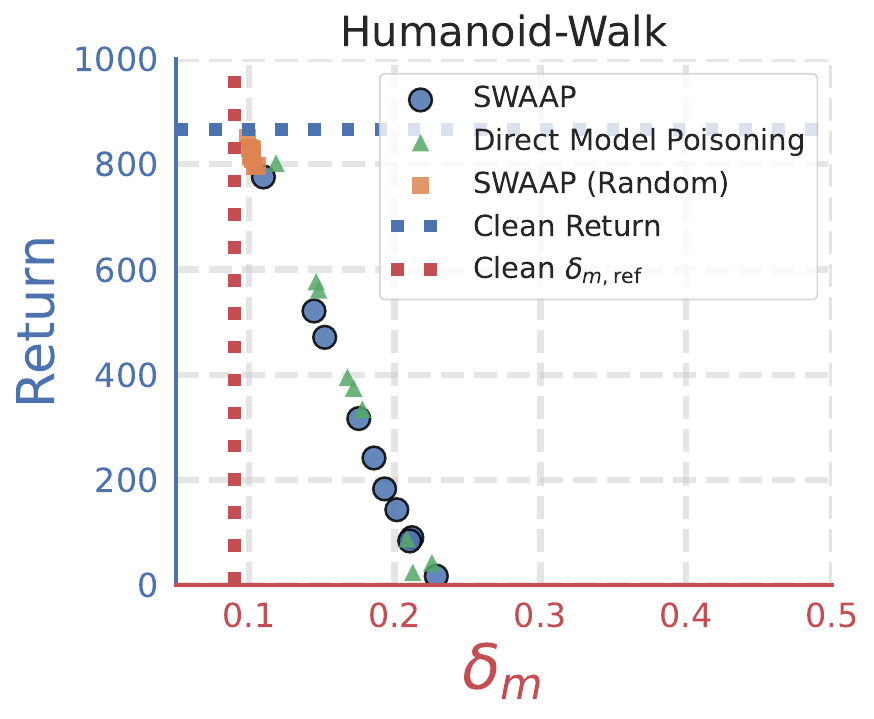}
        \subcaption{Humanoid-Walk frontier}
        \label{fig:humanoid_walk_scatter}
    \end{subfigure}
    \hfill
    \begin{subfigure}[t]{0.31\linewidth}
        \centering
        \includegraphics[width=\linewidth]{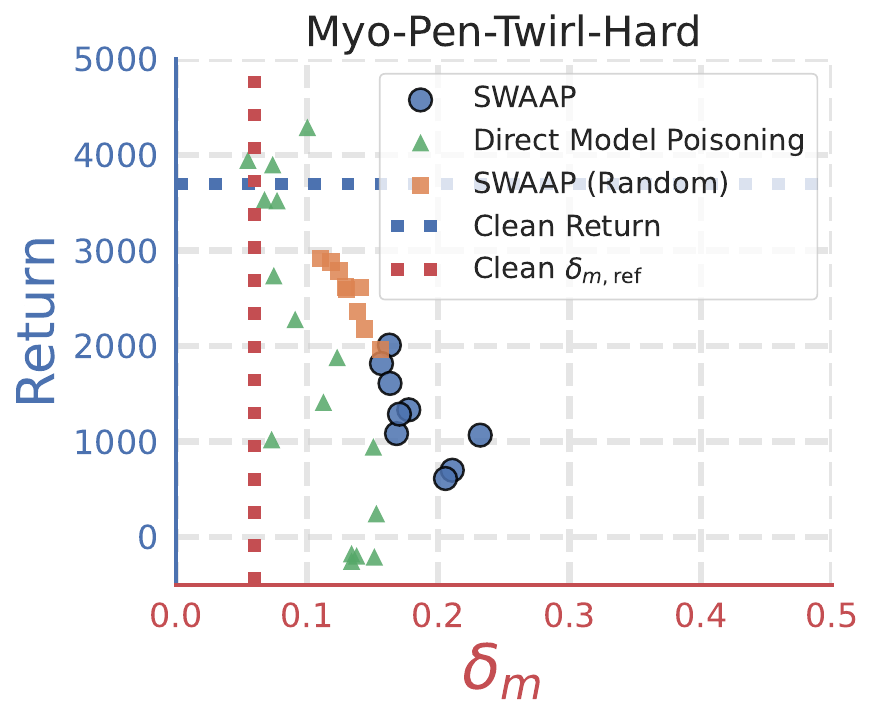}
        \subcaption{Myo-Pen-Twirl-Hard frontier}
        \label{fig:myo_pen_twirl_scatter}
    \end{subfigure}
    \hfill
    \begin{subfigure}[t]{0.31\linewidth}
        \centering
        \includegraphics[width=\linewidth,height=4cm]{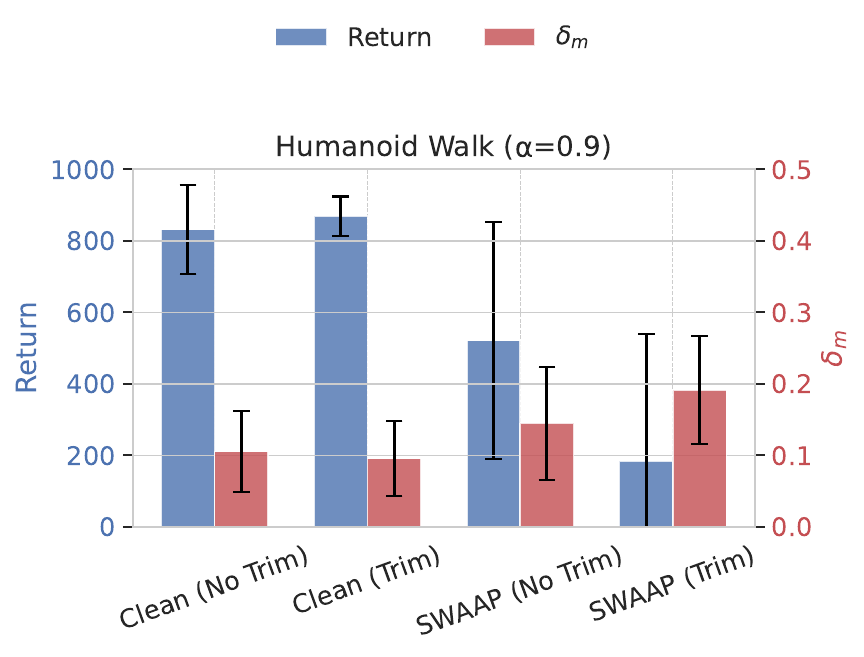}
        \subcaption{TRIM defense}
        \label{fig:trim}
    \end{subfigure}
    
    \caption{\small{
    (a) and (b) Attack-stealth frontier: DMP evaluates saved Stage~1 target models directly, while SWAAP applies Stage~2 data poisoning toward selected targets.
    (c) Robust-training defense: on Humanoid-Walk, SWAAP ($r_p,\alpha=0.1,0.9$) remains effective under TRIM robust training ($\beta=0.2$).}
    }
    \label{fig:scatter_comparison}
\end{figure*}

\vspace{-2ex}
\subsection{Ablations and Generalization}
\vspace{-1ex}
We provide additional ablations and generalization tests in Appendices~\ref{app:stage_bottleneck},~\ref{app:ablations},~\ref{app:generalization}. First, Appendix~\ref{app:stage_bottleneck} studies the relation between Stage~1 target identification and Stage~2 data realization. Direct Model Poisoning in Table~\ref{main_results} evaluates saved Stage~1 target models under an unrealistic model-overwrite threat model, while SWAAP realizes selected targets through bounded data poisoning. The gap between the two reflects the practical cost of realizing a target model only through fine-tuning data, rather than a failure of the two-stage decomposition. We also compare against DMP-data, which directly inserts Stage~1 target predictions without gradient matching; SWAAP is more effective, especially under TRIM, indicating that the gradient-matching realization step is useful.

Second, Appendices~\ref{app:pgd_baseline} and~\ref{app:top_rp} test Stage~2 design choices and alternative baselines, respectively. Replacing the top-$r_p$ transition-selection heuristic with random selection weakens the attack on Humanoid-Walk, supporting the choice of poisoning transitions most inconsistent with the Stage~1 target. On MW-Push, a PGD transition-poisoning baseline is weaker than SWAAP and produces larger data-level deviation, suggesting that local value-reducing perturbations are less effective and less stealthy than target-model gradient matching in this setting.

Finally, Appendices~\ref{app:trajectory_dp}--\ref{app:graybox} evaluate broader settings. With a fixed number of poisoned samples, increasing the fine-tuning buffer weakens but does not eliminate the attack. We also evaluate a trajectory-consistent variant in which a poisoned intermediate latent state is written consistently as both $s'_t$ and $s_{t+1}$ in sequential fine-tuning data. On Humanoid-Walk, this variant achieves $R=784\pm19$ and $\delta_m=0.111\pm0.003$ with $r_p=0.05$, while maintaining data-level reference deviation $\delta_d=.102\pm.007$. This suggests that SWAAP is not limited to unordered transition buffers, although the attack is weaker under the stricter sequential-consistency constraint. On DINO-WM Push-T, SWAAP reduces the success rate after Stage~2 poisoning, providing initial evidence across world-model backbones. In a preliminary gray-box transfer setting, poisoned data are constructed in raw observation space using a single-task surrogate and then applied to a multi-task TD-MPC2 victim with a different latent encoder. This does not replace the white-box setting as a worst-case robustness baseline, but suggests a path toward relaxing exact-parameter and shared-encoder assumptions.

\section{Conclusion and Limitations}\label{sec:conclusion}
We introduced SWAAP, a two-stage data poisoning framework that realizes harmful target dynamics through stealth-constrained fine-tuning data. Across continuous-control tasks and a range of evaluated defenses, SWAAP exposes a practical vulnerability in world-model adaptation pipelines and motivates stronger protection for world-model training data and learned dynamics.

As an initial study of world model poisoning, this work has limitations. First, our experiments primarily 
use TD-MPC2 as the victim model, with DINO-WM as an additional backbone; scaling to larger foundation models is left for future work. Second, SWAAP uses a two-stage approximation: Stage~1 identifies harmful, low-deviation target dynamics, and Stage~2 realizes them via gradient-matched data poisoning. Since some targets are easier to realize than others,  jointly optimizing target selection and data realization is an important next step. Third, we focus on untargeted performance degradation rather than targeted behavior manipulation. 
Finally, while our stealth evaluation covers representative 
detection and defense techniques, it does not rule out stronger 
detectors, deployment-specific defenses, or more stringent constrained-poisoning settings.


\clearpage
\bibliographystyle{plainnat}
\bibliography{references}

\newpage
\appendix
\clearpage
\startcontents[appendices]
\section*{Appendix Table of Contents}
\printcontents[appendices]{}{1}{}
\clearpage
\section*{Appendix}

\section{Use of LLMs}\label{app:llm_usage}

The authors used LLM-based assistants for manuscript preparation and presentation of supplementary materials, including grammar correction, wording suggestions, condensation of prose, LaTeX/formatting assistance, and improving the organization/readability of author-written code artifacts. LLMs were not used as part of the proposed method, theorem derivation, experiment design, result generation, or statistical analysis. All algorithms, experiments, figures, tables, citations, and technical claims were produced, checked, and approved by the authors. No confidential reviewer information, proprietary data, or sensitive personal data were used as LLM inputs.

\section{Broader Impact}\label{app:broader_impact}

This paper studies a new training-time attack surface for learned world models. The primary intended benefit is defensive: by showing that fine-tuning trajectories can manipulate learned dynamics and downstream planning, the work motivates stronger data validation, robust fine-tuning, and monitoring methods for world-model adaptation pipelines. These issues are especially important for embodied agents, robotics, autonomous systems, and other settings where world models may be updated from locally collected experience.

The same techniques could also inform misuse by suggesting ways to craft poisoned fine-tuning data. We mitigate this risk by studying the attack in controlled simulation benchmarks, limiting the threat model to bounded transition-target poisoning, and evaluating representative detection and robust-training defenses. We do not provide instructions for attacking deployed real-world systems. We believe that disclosure is warranted because the vulnerability arises naturally in continual fine-tuning pipelines, and understanding it is necessary for developing effective defenses.

More broadly, this work highlights that the robustness of world-model agents cannot be evaluated only at the final policy or planner level. The training data stream, learned dynamics, and downstream planning procedure are coupled, so future safety evaluations should consider how small changes in adaptation data can propagate into long-horizon behavior. We hope the proposed evaluations encourage further work on secure data collection, provenance tracking, anomaly detection, and robust world-model fine-tuning.

\section{Related Works}\label{app:related_works}
\subsection{World Models}
World models aim to learn compact and predictive representations of environment dynamics that can be used for planning and control. Instead of interacting directly with the environment, an agent can rely on its learned model to simulate trajectories, evaluate policies, and anticipate future states. First introduced in~\citep{ha2018world}, this work introduced a three-component architecture combining a variational autoencoder (VAE), a recurrent neural network (RNN), and a linear controller, demonstrating that policies trained entirely within a latent model can transfer back to the real environment. This work sparked a line of research exploring increasingly powerful and scalable model-based reinforcement learning frameworks.

Dreamer-style agents use latent dynamics models optimized end-to-end with reinforcement learning. DreamerV3~\citep{hafner2024masteringdiversedomainsworld} achieves state-of-the-art performance on visual control and robotic tasks. It employs a recurrent state-space model (RSSM) with deterministic state \(h_t\) and stochastic latent \(s_t\), performing rollouts entirely in latent space. Three heads are trained on the RSSM: a reward predictor \(\hat r_t\), a value function \(\hat V(\cdot)\), and an actor \(\pi\), using imagined rollouts. At decision time, the actor proposes candidate actions, either sampled or taken as the mean.

More recent approaches have investigated architectural advances tailored for large-scale and complex environments. DINO-WM~\citep{zhou2025dinowmworldmodelspretrained} leverages self-supervised vision transformers (ViTs) to improve perception quality, enabling stronger generalization in visually rich settings. Similarly, TD-MPC2~\citep{hansen2024tdmpc2} proposes a temporally abstracted model predictive control framework that combines world models with trajectory optimization, achieving sample-efficient learning and strong performance in high-dimensional continuous control tasks.

Parallel to these developments, diffusion-based world models have emerged as a promising alternative~\citep{pearce2023imitating,janner2022diffuser}. By parameterizing the transition distribution as a denoising diffusion process, these models can capture multi-modal and stochastic dynamics more effectively than conventional Gaussian latent models. Diffusion world models have been shown to improve both planning quality and robustness to uncertainty, making them attractive in settings where dynamics are highly non-deterministic.

\subsection{Model-Based Reinforcement Learning}
Model-based reinforcement learning (MBRL) is a powerful paradigm that improves sample efficiency and enables better generalization compared to purely model-free methods. In MBRL, an agent learns an explicit dynamics model of the environment and leverages this model for planning or policy optimization. We study an agent that leverages a learned world model to interact with an environment formalized as a Markov decision process (MDP) $(S, A, P, R, \gamma, \mu)$, where $\mu$ is the initial state distribution, $S$ is the state space, $A$ is the action space, $P: S \times A \to \Delta(S)$ is the transition kernel, $R: S \times A \times S \to \mathbb{R}$ is the reward function, and $\gamma \in (0,1)$ is the discount factor.  The goal is to find a policy $\pi: S \to \Delta(A)$ that maximizes the expected discounted return $\sum_{i=0}^{T}\gamma^iR(s_t,a_t)$. 

Unlike model-free approaches that learn value functions or policies directly from experience \citep{mnih2015human, lillicrap2016continuous, haarnoja2018soft}, MBRL explicitly learns a parametric approximation $\hat{P}_\psi$ of the transition kernel $P$, often using neural networks \citep{deisenroth2011pilco, chua2018deep, janner2019trust, hansen2024tdmpc2}. This learned dynamics model can then be used for model-predictive control (MPC) \citep{camacho2013model}, trajectory optimization \citep{tassa2012synthesis}, or to generate synthetic rollouts for policy improvement \citep{kaiser2019model, hafner2019learning}.

\subsection{Model Poisoning Attack}
Model poisoning attacks directly compromise a learned model by altering its parameters. They assume the adversary can craft malicious model updates to steer the learning outcome or directly replace model parameters. For example, carefully crafted local malicious updates can harm the performance of a federated learning system \citep{fang2020local}.

In reinforcement learning, model poisoning is particularly concerning when applied to the agent’s learned dynamics model. Small, adversarial modifications to the transition kernel can propagate over planning horizons, misleading policy improvement, and causing degraded performance. This makes model poisoning a uniquely severe threat in model-based RL, as even subtle deviations from the true dynamics can cascade into large errors in long-term decision making.

\subsection{Data Poisoning Attacks and Defenses}
Data poisoning attacks compromise learning systems by corrupting the training dataset to bias the learned model toward an adversarial objective. Unlike model poisoning, which manipulates parameters directly, data poisoning assumes the attacker can only influence the data stream but not the learning algorithm itself, which is a more realistic attack scenario. Data poisoning attacks have been shown to significantly degrade the model performance even with small amounts of poisoned data \citep{poisoning-SVM}. A recent and effective data poisoning technique leverages gradient matching, which optimizes poisoned samples such that their gradients closely align with those of a target adversarial objective \citep{geiping2021witchesbrewindustrialscale}. By ensuring that poisoned data induces updates similar to those of the adversary’s intended solution, gradient matching enables stealthy and highly effective poisoning even under limited attacker control. We also adopt gradient matching in the second stage of our attack methods, where we inject carefully crafted poisoned transitions into the newly collected fine-tuning dataset. 

In reinforcement learning, data poisoning is particularly dangerous since training relies on sequential interactions with the environment. By injecting corrupted transitions into the replay buffer or modifying observed trajectories, an adversary can degrade the long-term performance of the agent, or even embed targeted failures \citep{rakhsha2020policy, zhang2020adaptive}. 

Effective defenses against data poisoning attacks involve detection and training-time strategies. We consider two widely used data poisoning defenses in our work. \citet{chen2021depois} propose De-Pois, an attack-agnostic detection method that identifies poisoned data points with the help of a mimic model trained from clean data samples. By measuring the difference between the samples on the mimic model’s outputs, De-Pois can flag and remove suspicious points, improving robustness without assuming knowledge of the attack type. 

Complementing detection, training-time defenses aim to mitigate the effect of poisoned samples during learning. For example, \citet{poisoning-SVM} introduce the TRIM method and its iterative variant, which estimates and removes a fraction of potentially poisoned points based on deviation and the statistical properties of the data. 

\section{Theory and Derivations}\label{app:theory}

\subsection{Derivation of Transition Gradient}\label{app:theorem1}
\textbf{Theorem 3.1} \emph{The transition gradient of expected return in an MDP with transition dynamics $P_{\psi}$ 
and policy $\pi_{\theta}$ can be expressed by}
\begin{align*}
\nabla_{\psi} J(P_{\psi}, \theta) 
= \mathbb{E}_{(s,a,s') \sim P_{\psi}, \pi_{\theta}}
\Big[ (R(s,a,s') + \gamma V(s')) \nabla_{\psi}\log P_{\psi}(s'|s,a) \Big].
\end{align*}
\textbf{Proof:}
We start from the derivative of the state-value function of an arbitrary initial state $s_0$ and extend the state into the future time indefinitely, and expand it to write it as a recursive form in terms of $M$ function defined below, from which it unrolls to become an infinite series of the sum of $M(s_k)$ weighted by the probability of reaching $s_k$ from $s_0$ in $k$ steps.
\begin{align*}
\nabla_{\psi}V(s_0)
&= \nabla_{\psi}\sum_{a_0}\pi(a_0|s_0)Q(s_0,a_0) \\
&= \nabla_{\psi}\sum_{a_0}\pi(a_0|s_0)\sum_{s_1}P_{\psi}(s_1|s_0,a_0)[R(s_0,a_0,s_1)+\gamma V(s_1)] \\
&= \sum_{a_0}\pi(a_0|s_0)\sum_{s_1}\nabla_{\psi}P_{\psi}(s_1|s_0,a_0)[R(s_0,a_0,s_1)+\gamma V(s_1)] \\
&\quad+ \gamma \sum_{a_0}\pi(a_0|s_0)\sum_{s_1}P_{\psi}(s_1|s_0,a_0)\nabla_{\psi}V(s_1) \\
&= M(s_0) + \gamma \sum_{s_1}\rho_{\pi}(s_0\!\to\! s_1,k\!=\!1)\nabla_{\psi}V(s_1)
  && \rnote{a recursive relation about $\nabla_{\psi}V(s_i)$}\\
&= M(s_0) + \gamma \sum_{s_1}\rho_{\pi}(s_0\!\to\! s_1,1)\!\Big[M(s_1)+ \gamma  \!\sum_{s_2}\rho_{\pi}(s_1\!\to\! s_2,1)\nabla_{\psi}V(s_2)\Big] \\
&= M(s_0) +\sum_{s_1}\gamma \rho_{\pi}(s_0 \!\to\! s_1,1) M(s_1) + \sum_{s_2}\gamma^2 \rho_{\pi}(s_0\!\to\! s_2,2)M(s_2) + \cdots
  && \rnote{unrolling into future steps indefinitely}\\
&= \sum_x\sum_{k=0}^{\infty}\gamma^k \rho_{\pi}(s_0\!\to\! x,k)M(x)
  && \rnote{swap order of the summation and re-indexing}\\
&= \sum_s \eta(s) M(s)
  && \rnote{let $\eta(s)=\sum_{k=0}^{\infty}\gamma^k  \rho_{\pi}(s_0\!\to\! s,k)$} \\
&\propto \sum_s \frac{\eta(s)}{\sum_s \eta(s)} M(s)
  && \rnote{$\frac{\eta(s)}{\sum_s \eta(s)}$ is the stationary distribution of $s$}\\
&=\sum_s d_{\pi}(s)\sum_a\pi(a|s)\sum_{s'}\nabla_{\psi}P_{\psi}(s'|s,a)(R(s,a,s')+\gamma V(s')).
\end{align*}

where $M(s_0)\coloneqq\sum_{a_0}\pi(a_0|s_0)\sum_{s_1}\nabla_{\psi}P_{\psi}(s_1|s_0,a_0)[R(s_0,a_0,s_1)+\gamma V(s_1)]$, and the scaling from dividing the normalization constant $\sum_s \eta(s)$ can be absorbed into the learning rate, and \(\rho_{\pi}(s\rightarrow x,k=1) \coloneqq \sum_a \pi(a \vert s) P(x \vert s, a)\) is the transition probability of reaching $x$ from $s$ at step $k=1$, and the relation $\rho^\pi(s \to x, k+1) = \sum_{s’} \rho^\pi(s \to s’, k) \rho^\pi(s’ \to x, 1)$ is used to write the transition gradient into recursive form.
Therefore,
\begin{align*}
    \nabla_{\psi}J(P_{\psi},\pi) &= \sum_s d_{\pi}(s)\sum_a\pi(a|s)\sum_{s'}\nabla_{\psi}P_{\psi}(s'|s,a)(R(s,a,s')+\gamma V(s'))\\
    &= \mathbb{E}_{(s,a,r,s')\sim P_{\psi}, \pi}[(R(s,a,s')+\gamma V(s'))\nabla_{\psi}\log P_{\psi}(s'|s,a)].
\end{align*}

For planning with a finite horizon $T$, given a time-indexed value function $V_t(s_t)$ with terminal value $V_T(s_T)=0$, the unrolling stops at finite steps. We begin a recursive form of $\nabla_\psi V_t(s_0) $ that is similar to the infinite horizon one derived above
\begin{align*}
\nabla_\psi V_t(s_0) 
&= M_t(s_0) + \gamma \sum_{a_0}\pi(a_0|s_0)\sum_{s_1}P_\psi(s_1|s_0,a_0)\nabla_\psi V_{t+1}(s_1) \\
&= M_t(s_0) + \gamma \sum_{s_1}\rho_{\pi}(s_0\!\to\! s_1,k\!=\!1) \nabla_\psi V_{t+1}(s_1) \\
&= M_t(s_0) + \gamma \sum_{s_1}\rho_{\pi}(s_0\!\to\! s_1,k\!=\!1) \Big[M_{t+1}(s_1) + \gamma \sum_{s_2}\rho_{\pi}(s_0\!\to\! s_2,k\!=\!2)V_{t+2}(s_2)\Big] \\
&= M_t(s_0) + \gamma \sum_{s_1}\rho_{\pi}(s_0\!\to\! s_1,k\!=\!1) M_{t+1}(s_1) + ... + \gamma^{k'} \sum_{s_k}\rho_{\pi}(s_0\!\to\! s_{k'},k\!=\!k') \nabla_\psi V_{t+k'}(s_{k'}) \\
&= \sum_{i=0}^{k'-1}\gamma^i \rho_{\pi}(s_0\!\to\! s_i,i) M_{t+i}(s_i) + \gamma^{k'} \sum_{s_k}\rho_{\pi}(s_0\!\to\! s_{k'},k\!=\!k') \nabla_\psi V_{t+k'}(s_{k'})   \\
&= \sum_{i=0}^{T-t-1}\gamma^i \rho_{\pi}(s_0\!\to\! s_i,i) M_{i}(s_i)
\end{align*}
where in the last step, setting $k'=T-t$ and using $V_T(s_T)=0$ drops the last term, and therefore the finite horizon transition gradient is
\begin{equation*}
    \nabla_\psi J_T(P_\psi, \pi) = \nabla_\psi V_0(s_0)=\mathbb{E}_{(s_t,a_t,r_t,s_{t+1})\sim P_{\psi}, \pi}\big[ \sum_{t=0}^{T-1}\gamma^t\big(R(s_t,a_t,s_{t+1})+\gamma V_{t+1}(s_{t+1})\big)\nabla_{\psi}\log P_{\psi}(s_{t+1}|s_t,a_t)\big]
\end{equation*}

For the deterministic transition $s_{t+1}=f_\psi(s_t,a_t)$ in finite horizon

\begin{align*}
\nabla_\psi V_t(s_0)
&= \nabla_\psi \mathbb{E}_{a\sim\pi(\cdot|s)}\big[R(s_0,a_0,s_1)+\gamma V_{t+1}(s_1)\big] \\
&= \mathbb{E}_{a\sim\pi(\cdot|s)}\big[\nabla_{s_1}(R(s_0,a_0,s_1)+\gamma V_{t+1}(s_1))|_{s_1=f_\psi(s_0,a_0)}\nabla_\psi f_\psi(s_0,a_0,s_1)+ \gamma \nabla_\psi V_{t+1}(s_1) \big] \\
&= \mathbb{E}_{a\sim\pi(\cdot|s)}\big[\sum_{i=0}^{T-t-1}\gamma^t\nabla_{s'}(R(s_i,a_i)+\gamma V_{t+i+1}(s_{i+1}))|_{s_{i+1}=f_\psi(s_i,a_i)}\nabla_\psi f_\psi(s_i,a_i,s_{i+1})\big] 
\end{align*}

\subsection{BOME Details for Stage 1 Model Poisoning}\label{app_bome}
\begin{algorithm}
\small
\caption{Algorithm for Perturbed Model Identification}
\label{model_identification}
\begin{algorithmic}[1]
\State \textbf{Input:} env. dynamics $P$, dynamics model $\psi_0$, policy $\theta_0$, reward model $R$, value net $V_\phi$
\State \textbf{Output:} perturbed model $\hat{\psi}$
\While{$(\psi_k,\theta_k)$ not converged}
    \State Collect rollouts from $(\psi_k,\theta_k)$; use $R$ to form TD targets; update $V_\phi$
    \State $\theta_k^{0} \gets \theta_k$
    \For{$i \gets 1$ to $W$}
        \State $\theta_k^{i+1} \gets \theta_k^{i} + \nabla_{\theta} J(\psi_k,\theta_k^{i})$
    \EndFor
    \State Collect rollouts to form replay buffers $\mathcal{B}_{kW}$, $\mathcal{B}_{pk}$, $\mathcal{B}_{kk}$ from $(\psi_k,\theta_k^{W})$, $(P,\theta_k)$, $(\psi_k,\theta_k)$, respectively
    \State $\Omega \gets [\,]$
    \For{$i \gets 1$ to $N$}
        \State Sample $\mathcal{B}_{kW}$ to compute $\nabla_{\psi} J(\psi_k,\theta_k^{W})$
        \State Sample $\mathcal{B}_{pk}$ to compute $\nabla_{\theta} J(P,\theta_k)$, $\nabla_{\theta}L(P_{\psi_k}, P)$ and $\nabla_{\psi} L(P_{\psi_k},P)$ 
        \State Sample $\mathcal{B}_{kk}$ to compute $\nabla_{\psi} J(\psi_k,\theta_k)$ and $\nabla_{\theta} J(\psi_k,\theta_k)$
        \State Compose $\nabla f(\psi_k,\theta_k)$ and $\nabla q(\psi_k,\theta_k)$ from Eqs.\ref{eqn_grad_f} and \ref{eqn_grad_q}, and $\lambda_k$
        \State Append $\nabla f(\psi_k,\theta_k) + \lambda_k \nabla q(\psi_k,\theta_k)$ to $\Omega$
    \EndFor
    \State $\omega_k \gets \frac{1}{N}\sum_{\omega\in\Omega}\omega$
    \State $(\psi_{k+1},\theta_{k+1}) \gets (\psi_k,\theta_k) - \xi\cdot \omega_k$
\EndWhile
\State \textbf{return} $\hat{\psi}\gets \psi_k$
\end{algorithmic}
\end{algorithm}
{The adversary minimizes the return of the surrogate policy that learns from perturbed transitions by solving the following bilevel optimization problem: 
\begin{equation}
\begin{array}{ll}
\min_{\psi} & J(P, \theta(\psi)) + \lambda L(P_\psi, P)  \\
\text { s.t. }  & \theta(\psi)\in \arg \max_{\theta'} J(P_{\psi}, \theta'). 
\label{eqn_bo_orig}
\end{array}
\end{equation}
The goal is to minimize the constraint-regularized expected return with respect to the optimal policy derived from the poisoned environment.}

This bilevel optimization problem remains challenging due to its nested structure. 
The primary difficulty lies in computing the derivative $\nabla_{\psi}\theta(\psi)$. 
If one directly treats $\theta(\psi)$ as a function of $\psi$, gradient-based methods such as 
hypergradient descent \citep{bard2013practical} update $\psi$ using $\nabla_{\psi}J(P, \theta) = \nabla_{\psi}\theta(\psi)\nabla_{\theta} J(P, \theta),$ but evaluating $\nabla_{\psi}\theta(\psi) = \nabla_{\psi}\arg \max_{\pi_{\theta'}} J(P_{\psi}, \theta')$ 
requires solving linear systems and computing costly second-order Hessians. 
Prior work has attempted to bypass this difficulty either through surrogate approximations 
or heuristics that are difficult to control \citep{pedregosa2016hyperparameter, ghadimi2018approximation, mackay2019self}. 
Moreover, reinforcement learning objectives are highly non-convex, involve high-dimensional state--action spaces, 
and present a complex optimization landscape. As a result, naive bilevel optimization methods often fail to converge reliably, 
making them unsuitable for Problem~\ref{eqn_bo_orig} \citep{liu2021investigating}.

To obtain a more scalable solution, we adopt the 
{first-order dynamic barrier gradient
descent method (BOME) in~\citep{liu2022bome}},
which reformulates the problem by replacing the implicit argmin operator with a value-function constraint. 

This allows the outer objective and inner objective to be solved without explicitly computing the derivative $\nabla_{\psi}\theta(\psi)$. Formally, let $f(\psi, \theta) \coloneqq J(P, \theta) + L(P_\psi, P)$ denote the outer objective, and $q(\psi, \theta) \coloneqq \max_{\pi_{\theta'}} J(P_{\psi}, \theta') - J(P_{\psi}, \theta)$, which measures the suboptimality of $\pi_\theta$ relative to the optimal policy under $P_\psi$. Under the value-function approach, the bilevel problem becomes the following constrained optimization:
\[
\min_{\psi, \theta} f(\psi, \theta) \quad \text{s.t.} \quad q(\psi, \theta) \leq 0.
\]

This transformed problem is solvable by iteratively updating $(\psi, \theta)$ to decrease $f$ while at the same time keeping the constraint $q \leq 0$ satisfied by decreasing $q$ whenever $q>0$ in each step:
\begin{align}
    (\psi_{k+1}, \theta_{k+1}) &\leftarrow (\psi_k, \theta_k) - \xi \omega_k\\
    \text { where }  \quad  \omega_k &= \arg \min_{\omega} ||\nabla f(\psi_k, \theta_k)-\omega||^2 \\
    \text{ s.t. } &\langle \nabla q(\psi_k, \theta_k), \omega \rangle \geq \phi_k
\end{align}

this could be solved in closed form, which gives $\omega_k = \nabla f(\psi_k, \theta_k) + \lambda_k \nabla q(\psi_k, \theta_k)$, with $\lambda_k=\max \left(\frac{\phi_k-\left\langle\nabla f\left(\psi_k, \theta_k\right), \nabla q\left(\psi_k, \theta_k\right)\right\rangle}{\left\|\nabla q\left(\psi_k, \theta_k\right)\right\|^2}, 0\right)$ and $\phi_k$ is chosen to be $\eta q(\psi,\theta)$ or $\eta \left\|\nabla q(\psi,\theta)\right\|^2$ (we used $\eta \left\|\nabla q(\psi,\theta)\right\|^2$ in experiment and take $\eta=0.5$). Therefore, the procedure to optimize $f$ by jointly updating $(\psi, \theta)$ is:
\begin{equation}
    (\psi_{k+1}, \theta_{k+1}) \gets (\psi_k, \theta_k) - \xi [\nabla f(\psi_k, \theta_k) + \lambda_k \nabla q(\psi_k, \theta_k)]
\end{equation}
where $\nabla f(\psi_k, \theta_k) = \nabla_{(\psi_k,\theta_k)} f(\psi_k, \theta_k)$ is the gradient update of outer problem and $\nabla q(\psi, \theta)=\nabla_{(\psi,\theta)} q(\psi, \theta)$ imposes the constraint. Expressed explicitly, the gradient of $f$ and $q$ are:
\begin{align}
\nabla_{(\psi,\theta)} f(\psi_k, \theta_k)
&=
\left(
\lambda \nabla_{\psi}L(P_{\psi_k}, P),
\nabla_{\theta} J(P, \theta_k) + \lambda \nabla_{\theta}L(P_{\psi_k}, P)
\right),
\label{eqn_grad_f} \\
\nabla_{(\psi,\theta)} q(\psi_k, \theta_k)
&\approx
\left(
\begin{aligned}
&\nabla_{\psi}J(\psi_k, \theta^W_k)
- \nabla_{\psi}J(\psi_k, \theta_k), \\
&- \nabla_{\theta} J(\psi_k, \theta_k)
\end{aligned}
\right).
\label{eqn_grad_q}
\end{align}

where we use the shorthand $J(\psi_k, \theta_k) \coloneqq J(P_{\psi_k}, \theta_k)$, 
and $\theta^W_k$ is the $W$-step approximation of the optimal policy 
$\theta^* \in \arg \max_{\theta'} J(P_{\psi}, \theta')$. 

\subsection{Stochastic Transition Functions Case}\label{app:stochastic}
The stochastic world model minimizes the negative log likelihood to fine-tune on data $D$
\begin{equation}
  \mathcal{L}(\psi; D)
= - \sum_{(s,a,s') \in D} \log P_{\psi}(s' \mid s,a).  
\end{equation}

where $s'\sim P(\cdot|s,a)$, and $P$ is the true environment transition, the loss function used in the bilevel optimization is defined as 
\begin{equation}
L(P_\psi, P) = 
    \mathbb{E}_{(s,a)\sim P, \pi_{\theta(\psi)}}\big[D_{\mathrm{KL}}(P(\cdot|s,a)\| P_{\psi}(\cdot|s,a))\big]
\end{equation}

For \emph{stochastic} transitions, we define the residual as the KL divergence between the empirical next-state distribution induced by samples $(s,a,s') \in D$ and the model distribution:
\begin{equation}
e_{\psi}(s,a) = D_{\mathrm{KL}}\!\left(P(\cdot \mid s,a)\middle\|P_{\psi}(\cdot \mid s,a)\,\right),
\end{equation}

For \emph{stochastic} transition models, the gradients we defined in stage 2 are shown below
\begin{equation}
\begin{aligned}
&G_{\text{real}}
=
\mathbb{E}_{(s,a)\sim \tilde D}
\Big[
\nabla_{\psi_0}
D_{\mathrm{KL}}\!\left(P_{\psi_0}(\cdot \mid s,a)\middle\|\hat{P}_{\tilde{D}}(\cdot \mid s,a)\,\right)
\Big], \\
&G_{\text{target}}
=
\mathbb{E}_{(s,a)\sim D_{\text{all}}}
\Big[
\nabla_{\psi_0}
D_{\mathrm{KL}}\!\left(
P_{\psi_0}(\cdot \mid s,a)
\,\middle\|\,
P_{\hat\psi}(\cdot \mid s,a)
\right)
\Big].
\end{aligned}
\end{equation}
where we use $\hat{P}_{D}(\cdot \mid s,a)$ denotes the empirical distribution of next states associated with $(s,a)$ in the dataset $D$, in which given a $(s,a)$ pair an agent can sample multiple $s'$ using environment dynamics or its transition models, establishing a distribution.

The loss function of Stage 2 is defined as below.
\begin{equation}
\small
\begin{aligned}
L_P^{\text{sto}}
&=(1-\alpha)\big(1-\cos(G_{\text{real}},G_{\text{target}})\big) +\alpha\!\!\sum_{(s_i,a_i)\in D_p}
\Big|
D_{\mathrm{KL}}\!\big(P(\cdot\mid s_i,a_i)\,\|\,\hat P_{D_p}(\cdot\mid s_i,a_i)\big) -
D_{\mathrm{KL}}\!\big(P(\cdot\mid s_i,a_i)\,\|\,P_{\hat\psi}(\cdot\mid s_i,a_i)\big)
\Big|.
\end{aligned}
\label{dp_objective_sto}
\end{equation}

For stochastic world models, we define fine-tune data-level deviation $\delta_d = \frac{1}{|\tilde{D}|} \sum_{(s_i,a_i)\in\tilde{D}}
D_{\mathrm{KL}}\!\big(
P(\cdot|s_i,a_i)
\,\|\,\hat P_{\tilde{D}}(\cdot | s_i,a_i)
\big)$ and test-time model-level deviation $\delta_m = \frac{1}{T} \sum_{i=0}^{T}
D_{\mathrm{KL}}\!\big(
P(\cdot | s_i,a_i)
\,\|\,\hat P_\psi(\cdot | s_i,a_i)
\big)$
where $T$ is the number of observed transitions.

\section{Implementation Details}\label{app:implementations}
We instantiate the experiment based on the TD-MPC2 framework. Thus all neural networks including policy, dynamics model, encoder and value function have MLP architecture with several layers, for detailed information about architecture, please refer to~\citep{hansen2024tdmpc2}.
\subsection{Model Predictive Control}\label{mpc_details}


\begin{algorithm}[H]
\caption{General Model Predictive Control (MPC)}
\label{mpc_algorithm}
\begin{algorithmic}[1]
    \State \textbf{Input:} World model $\psi$, current state $s_t$, goal $o_g$ (optional) horizon $H$, number of rollouts $\texttt{num\_samples}$
    \State Encode: $z_t = \text{enc}(s_t)$, $z_g =\text{enc}(o_g)$ if goal is given
    \For{each MPC step}
    \State Sample $\texttt{num\_samples}$ candidate action sequences $\{a_{0:H-1}^i\}_{i=1}^{\texttt{num\_samples}}$ (from a policy 
    \Indent
        prior, an action sampler, or learned actor)
    \EndIndent
    \For{each sequence $i$}
        \State Roll out latent trajectory $\hat z_{1:H}^i = P_\psi(\hat z_t, a_{0:H-1}^i)$
        \State Evaluate cost or return:
        \[
        J^i =
        \begin{cases}
            \sum_{h=0}^{H-1} \gamma^h r(\hat z_h^i, a_h^i) + \gamma^H Q(\hat z_H^i) & \text{(value/bootstrap)} \\
            \|\hat z_H^i - z_g\|^2 & \text{(goal-closeness)} \\
            \text{other task-specific objective} &
        \end{cases}
        \]
    \EndFor
    \State Select best sequence(s) according to $J^i$
    \State Refit a sampling distribution to top-$K$ sequences
\EndFor
\State \textbf{return} first action or first $k$ actions from selected sequence or sampler
\end{algorithmic}
\end{algorithm}

We provide a general formulation of Model Predictive Control (MPC), which abstracts across variants such as TD-MPC2~\citep{hansen2024tdmpc2}, DINO-WM~\citep{zhou2025dinowmworldmodelspretrained}. In this section, we occasionally use $\hat{z}$ to denote an imagined latent, compared to the real latent $z$ encoded from the real state $s$ from the environment.

To model the dynamics, TD-MPC2~\citep{hansen2024tdmpc2} learns the world model $P_{\psi}(z'|z,a)$ that predicts next latent state given current latent state and action, where $z=enc(s)$ is the latent state from a learned encoder, and its agent plans by using the learned world model $P_{\psi}$ together with a learned prior policy \(\pi_{\mathrm{prior}}\). At each planning step, the planner first obtains \texttt{num\_pi\_trajs} candidate trajectories by rolling out \(\pi_{\mathrm{prior}}\) under transition $P_{\psi}$; these yield the policy-guided action sequences. The planner then samples an additional \(\texttt{num\_samples}-\texttt{num\_pi\_trajs}\) random action sequences from a stochastic proposal (e.g., i.i.d. or Gaussian-perturbed sequences). During the inner improvement/refit loop, the action components coming from the prior policy are treated as guidance and are typically held fixed, while only the remaining (random) action sequences are optimized. This hybrid design reduces search dimensionality and biases search toward plausible behavior while still allowing exploratory refinement.

DINO-WM~\citep{zhou2025dinowmworldmodelspretrained} does not use a learned prior. Instead, it draws many random action sequences, rolls each sequence forward in the world model $P_{\psi}$ for horizon $H$, and scores each rollout by a final-state objective (the distance to a goal latent $z_g$, e.g., $J = \|\hat z_H-z_g\|_2$). The top-performing sequences are retained as elites and the sampling distribution is refit to those elites (CEM-style); this process repeats until the sampling distribution concentrates on sequences that reach the goal.

The above MPC algorithm highlights the shared structure: rolling out candidate 
action sequences over a planning horizon $H$ using a learned world model 
$P_\psi$, evaluating them under a task-specific objective (e.g., 
bootstrapped return, goal closeness, or other criteria), and executing 
the first action of the best sequence. The number of sampled rollouts 
$\texttt{num\_samples}$ and the planning horizon $H$ are hyperparameters 
that directly affect the quality of planning and computation cost. 
The general MPC procedure is summarized in Algorithm~\ref{mpc_algorithm}.

\begin{table}[!t]
\centering
\small
\caption{Key hyperparameters used in experiments; other parameters are the same as in TD-MPC2}
\label{tab:hyperparams}
\begin{tabularx}{\linewidth}{%
  >{\raggedright\arraybackslash}p{2.6cm}
  >{\centering\arraybackslash}p{3.2cm}
  X
}
\toprule
\textbf{Hyperparameter} & \textbf{Typical value(s)} & \textbf{Description} \\
\midrule
\multicolumn{3}{l}{\textbf{Model Poisoning Hyperparameters}} \\[4pt]

\text{obs\_dim}  &
17-223 &
The range of observation dimensions for tasks \\

\text{action\_dim}  &
6-39 &
The range of action dimensions for tasks  \\

\text{latent\_dim}  &
512 &
The latent state dimension to encode observation state \\
H  &
3 &
Planning horizon length, which is the number of rollout steps into the future when planning\\

\text{num\_samples}  &
512 &
Total number of trajectories for planning, including policy trajectories and random trajectories\\

\text{policy\_samples}  &
24 &
number of trajectories for planning that are from policy prior\\

\text{iterations}  &
6 &
number of iterations done to optimize the actions in planning\\

\(\lambda\)  &
\(\{0,\,1,\,10,\,100\}\) &
Consistency coefficient applied to the $L(P,P_\psi)$ term during model poisoning in \Eqref{eqn_bo_orig} \\[6pt]

\(W\) &
\(\{30\}\) &
The number of update steps used to approximate $\theta^*_k$ \\

\(N\)  &
\(\{16\}\) &
The size of $\Omega$, the number of small updates aggregated to compute gradients of $f_k,q_k$ and $\lambda_k$\\

\(\text{sample\_batch\_size}\) &
\(\{256\}\) &
Number of transitions used to calculate each sample in $\Omega$\\

\(\text{num\_step}\) &
\(\{100,500\}\) &
The imagined rollout length, 500 for DMControl tasks and 100 for others. \\

\(\text{buffer\_size}\) &
\(\{500\}\) &
Buffer size for every buffer used in Algorithm~\ref{model_identification}. \\

\midrule
\multicolumn{3}{l}{\textbf{Data Poisoning Hyperparameters}} \\[4pt]

\(r_p\) &
\(\{0.1,\,0.2,\,0.3\}\) &
Fraction of the dataset (or proportion of trajectories/transitions) modified by adversary. \\

\(\alpha\)  &
\(\{0.1,\,0.5,\,0.9, \,0.95,\,0.99\}\) &
Data poisoning regularization coefficient in \Eqref{dp_objective}\\[6pt]

\(\text{poison\_steps}\) &
\(\{5000\}\) &
The number of gradient matching update steps. \\

\(\text{step\_size}\) &
\(\{10\}\) &
Step size to update the poisoned samples in gradient matching. \\

\(\text{noise\_scale}\) &
\(\{0.1,0.3,0.5\}\) &
The random perturbation scale of SWAAP (Random) \\

\(\text{train\_epochs}\) &
\(\{100,500\}\) &
The number of training epochs for fine-tuning, 100 for DMControl tasks and 500 for others. \\

\(\text{learning\_rate}\) &
\(\{0.01, 0.0001\}\) &
Learning rate of fine-tuning, 0.01 for MyoSuite and 0.0001 for others. \\

$|D|$ &
\(\{5000\}\) &
Size of the fine-tuning dataset. \\

$|D_\text{all}|$ &
\(\{50000\}\) &
Size of the dataset to approximate $G_\text{target}$ \\

\bottomrule
\end{tabularx}
\end{table}

\subsection{Hyperparameters}\label{app:parameters}
We report our SWAAP key hyperparameters in Table~\ref{tab:hyperparams}. Other parameters are the same as stated in TD-MPC2~\citep{hansen2024tdmpc2}. 

\subsection{Stage 1: Target Model Identification}\label{app:model poisoning implementation}

The model poisoning starts with a trained TD-MPC2 checkpoint (such as humanoid-walk-2.pt from TD-MPC2 model base), from which following components are extracted as input to Algorithm~\ref{model_identification}, the trained policy network $\pi_\theta$ as $\theta_0$, the dynamics model $P_\psi$ is used as $\psi_0$, the value network (a $Q$ net in TD-MPC2), reward model, and an encoder, where the reward model and encoder will be freezed during model poisoning for there are no more information about the reward and encoding to be gained in our adversarial training. The model-based agent consists of a prior policy $\theta$ and transition model $\psi$, and $\pi_{\theta(\psi)}$ stands for the agent's policy that it takes after conducting MPC planning in imagined trajectories generated from the prior policy $\theta$ and transition $\psi$.

During each iteration of update, the trajectory rollouts are collected from $(\psi_k, \theta^W_k)$, $(P, \theta_k)$, $(\psi_k, \theta_k)$ to form $\mathcal{B}_{kW}$, $\mathcal{B}_{pk}$, $\mathcal{B}_{kk}$, which means $\mathcal{B}_{kW}$ is formed by trajectories generated by MPC policy $\pi_{\theta_{k}^W}(\psi_k)$ and transition $\psi_k$, $\mathcal{B}_{pk}$ is formed by trajectories generated by MPC policy $\pi_{\theta_k}(\psi_k)$ and environment transition $P$, and $\mathcal{B}_{kk}$ is formed by trajectories generated by MPC policy $\pi_{\theta_k}(\psi_k)$ and transition $\psi_k$. Using the reward model, the value network is updated on trajectories collected from $(\psi_k, \theta_k)$ by minimizing the TD error.

Given data $(s,a,s')$, the policy return gradient $\nabla_{\theta_k} J(\psi,\theta)$ is computed using the TD-MPC2 update\_pi function, which varies the action to maximize the value estimate $Q(s,a_{pred})$ and a scaled action entropy at the same time, where $a_{pred}$ is sampled from prior policy $\theta_k$ given state $s$. The transition return gradient $\nabla_{\psi_k} J(\psi,\theta)$ is computed by varying the prediction $s'_{pred}=P_\psi(s,a)$ to minimize the value estimate of predicted next state $Q(s'_{pred}, a'_{pred})$ where $a'_{pred}$ is sampled from prior policy $\theta_k$ given state $s'_{pred}$.

During training of Stage $1$, we evaluate $J(P,\pi_{\theta_0}(\psi_k))$, the return of an unmodified policy using the perturbed world model $\psi_k$ to plan in real environment $P$, for $10$ episodes every $10$ iterations, producing figures like Fig.~\ref{fig:model_poisoning_training_curves}, from which we can select the identified perturbed world model based on the testing performance (ideally one with low return while $\delta_m$ is also low), and this identified dynamics model will be used to propose the perturbed next state in Stage $2$.

\subsection{Stage 2: Data Poisoning}
The first steps in Stage $2$ involves collecting two datasets, one large dataset of $50,000$ transitions (which translates to $100$ independent trajectories for DMControl tasks $500$ for all other tasks) for estimating the target gradient $G_{target}$, and one small dataset of $5,000$ transitions for $G_{real}$. During data collection, each action has a $3\%$ chance of being replaced by a random uniform action, which is done for data diversity and to cover enough state-action pairs. 

Given the fine-tune dataset, the deviations $e_\psi(s,a,s')$ of each transition are calculated and ranked, and the top $r_p$ transitions are selected to be poisoned, which are considered to be more vulnerable than other transitions.  To ensure the perturbation remains stealthy, after each gradient matching step, we project the perturbed state $\tilde{s}'$ back to the L2 norm ball around the original state $s'$, but since TD-MPC2 world model constrains the output prediction within a valid range by architectural design, this step is not needed in our main experiments.

Once the poisoned data are crafted, we choose the fine-tuning hyperparameters (train\_epochs and learning\_rate) so that fine-tuning the clean world model on clean data of the same size does not degrade performance. This ensures that the observed degradation is caused by poisoning rather than overfitting during fine-tuning.

SWAAP (Random) use the same setting in Stage $2$ like SWAAP, and the only difference comes from that the perturbed model $P_{\hat{\psi}}$ is not learned but crafted from applying a random perturbation of scale being $noise\_scale$ elementwise to the true transition, which is $\tilde{s}'= s'\bigodot (1+\epsilon)$ with $\epsilon\sim[-noise\_scale, noise\_scale]$, we report SWAAP (random) with $noise\_scale=0.3$ in Table~\ref{main_results} and in Appendix~\ref{app:ours_data}.

\subsection{Hardware Specification and Computational Cost}\label{app:hardware}
We conduct all experiments on a workstation with an Intel Core Ultra 9 285K CPU, an NVIDIA GeForce RTX 5080 GPU with 16GB VRAM, and 64GB system RAM.

Stage 1 takes approximately 1 hour on DMControl and MetaWorld and 2 hours on MyoSuite. Stage 2 takes roughly 10 minutes. Our attack designs intentionally reduce computational overhead: (1) Stage 1 uses a first-order dynamic barrier method, avoiding the expensive Hessian computations typically required in bilevel optimization; (2) Stage 2 employs a modified gradient-matching step that selects the most influential data points for poisoning instead of exhaustively searching over all possible combinations, which significantly reduces computational cost.

\section{Main Extended Results}\label{app:extended_main}

\subsection{SWAAP and SWAAP (Random)} \label{app:ours_data}
We show more results on SWAAP and SWAAP(Random) in Figure~\ref{swaap_more}, comparing the return and $\delta_m$ of various tasks under SWAAP, SWAAP(Random) with the clean return and clean $\delta_m$ at different poison ratios $r_p$. A higher $\alpha$ constrains model deviation more strictly, reducing attack strength but improving stealth. A lower $\alpha$ relaxes the deviation constraint and increases performance degradation. A higher data poisoning rate $r_p$ will also strengthen our attack performance.

Empirically, across all tested $\alpha$ and $r_p$, we did not observe gradient explosion or vanishing. Optimization remained numerically stable throughout Stage 2.

\begin{figure}
    \centering
    \includegraphics[scale=0.38]{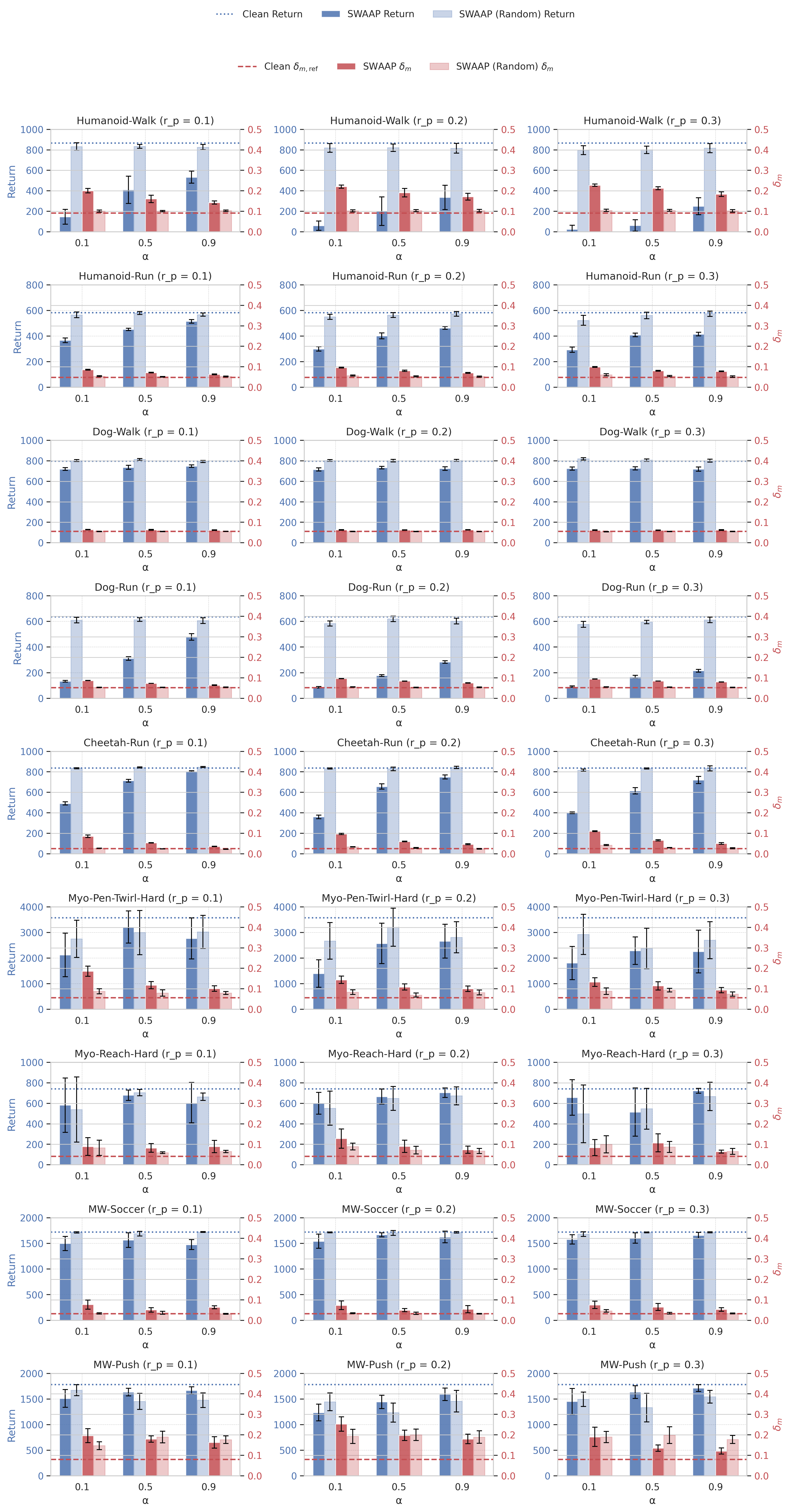}
    \caption{
    SWAAP and SWAAP (Random) across poisoning ratios $r_p$ and stealth weights $\alpha$. Blue bars show return and red bars show model-level deviation $\delta_m$. Lowering $\alpha$ often increases SWAAP's return degradation more than SWAAP (Random), showing that the Stage~1 target provides useful poisoning directions beyond random perturbation.
    }
    \label{swaap_more}
\end{figure}

\subsection{Model Poisoning Training Curve}\label{app:model_poisoning_training_curves}
We show the stage one model training curve in Figure~\ref{fig:model_poisoning_training_curves}, which comes from testing the agent that directly uses the perturbed model for $10$ episodes after every few iterations of updates. From these plots, we observe that $\lambda$ suppresses the model's deviation, but the varying influence of $\lambda$ across different environments implies that different $\lambda$ values are required to constrain the model's perturbation. 
\begin{figure}[!t]
    \centering
    \includegraphics[scale=0.3]{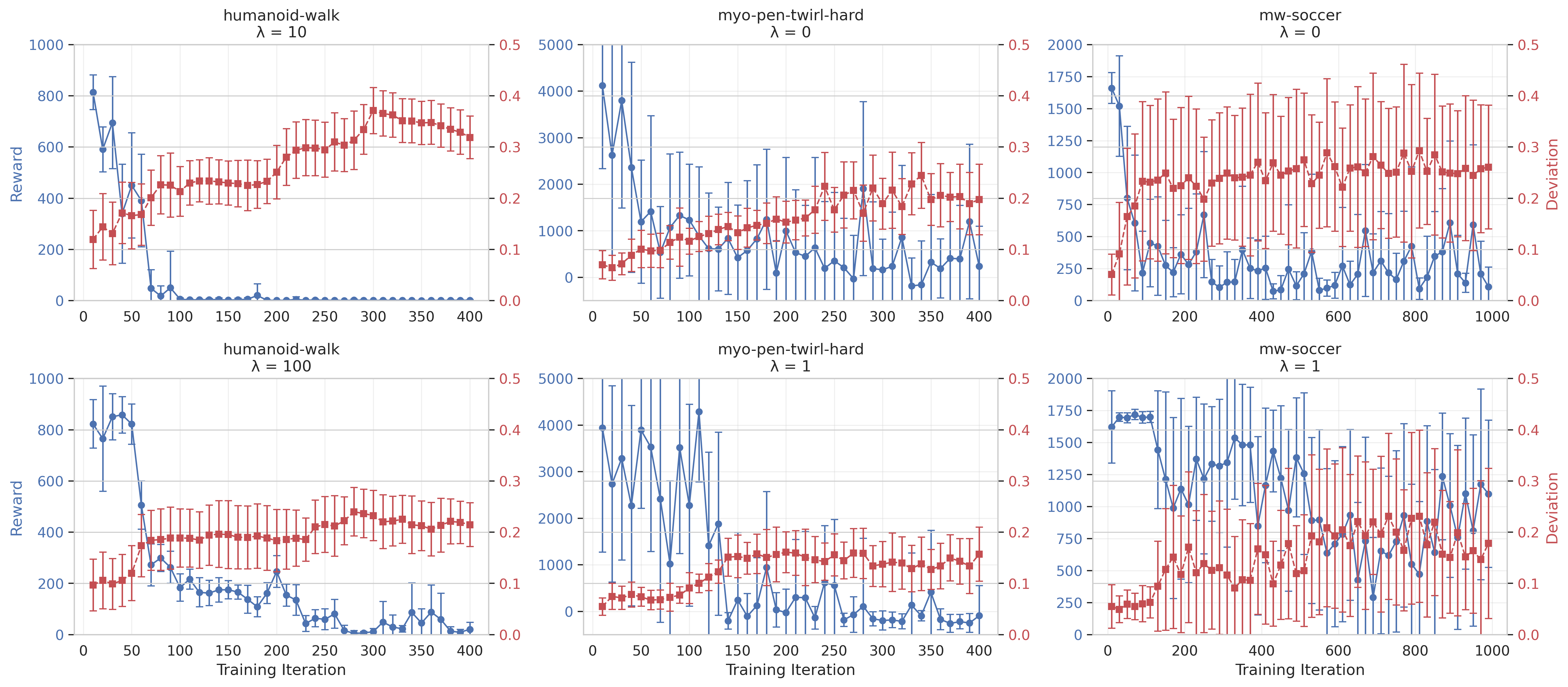}
    \caption{Model poisoning curves (referred to as Direct Model Poisoning in Table~\ref{main_results}) of return and $\delta_m$ tested across different training iterations under different $\lambda$, each testing is the average of $10$ test episodes, and thus each data point is a single run with $10$ episodes. From these plots, we observe that $\lambda$ is suppressing the deviation of the model, but the varying influence of $\lambda$ over different environments implies that a different $\lambda$ is required to constrain the perturbation of the model.}
    \label{fig:model_poisoning_training_curves}
\end{figure}

\subsection{Pre-Training Detection}\label{app:more_training_detection}

Figure~\ref{pretraining_detection} presents additional results on pre-training detection, showing the distribution of data-level deviation before and after poisoning under different values of the poisoning regularization coefficient $\alpha$. As $\alpha$ increases, the deviation of the poisoned transitions becomes more similar to those of clean data, making the attack more difficult to detect using deviation-based filtering.

\subsubsection{Residual-Based CUSUM Pre-Training Detector}
\label{app:cusum}

We evaluate a CUSUM-style sequential detector to screen candidate fine-tuning trajectories before they are used for model updates. Instead of using the latent-state norm $\|z_{t+1}\|_2$, we use a transition-residual statistic measured against the clean reference world model:
\[
x_t=\|z_{t+1}-P_{\psi_{\mathrm{ref}}}(z_t,a_t)\|_2 .
\]
This statistic is aligned with pre-training detection because a defender with a reference world model can test whether each observed transition is anomalous relative to the predicted next latent state.

Given clean calibration trajectories, we estimate a robust center and scale using the median $m$ and median absolute deviation $\mathrm{MAD}$ of the residual signal. For each candidate trajectory, we form centered residuals $r_t=x_t-m$ and update a one-sided CUSUM statistic
\[
S_t=\max(0,S_{t-1}+r_t-\kappa),
\qquad
\kappa=c_{\kappa}\mathrm{MAD}.
\]
We use the trajectory-level score $\max_t S_t$ for detection. The threshold is chosen from held-out clean trajectories to target a false-positive rate of $5\%$, and we report ROC-AUC, TPR at the clean-calibrated threshold, and the realized empirical FPR.

For Humanoid-Walk, we use 100 clean trajectories (50{,}000 transitions) to estimate the residual median/MAD and 10 separate held-out clean trajectories (5{,}000 transitions) to estimate empirical FPR; we compare against 10 poisoned trajectories. Because this held-out set contains only 10 clean trajectories, the empirical FPR is quantized in increments of $0.10$. For MW-Push, we use 500 clean trajectories (50{,}000 transitions) for calibration and 50 held-out clean trajectories (5{,}000 transitions) for empirical FPR; we compare against 50 poisoned trajectories.

\begin{table}[t]
\centering
\caption{
Residual-based CUSUM pre-training detection. The detector uses
$x_t=\|z_{t+1}-P_{\psi_{\mathrm{ref}}}(z_t,a_t)\|_2$
and scores each trajectory by $\max_t S_t$. Thresholds target FPR $5\%$ on held-out clean trajectories; empirical FPR differs slightly because of finite held-out sets.
}
\label{tab:cusum_residual}
\begin{tabular}{c|c|c|c|c|c}
\toprule
Environment & $\alpha$ & AUC & TPR & Emp. FPR & Mean residual \\
\midrule
Humanoid-Walk & 0.1  & 0.665 & 0.20 & 0.10 & 0.4034 \\
Humanoid-Walk & 0.5  & 0.530 & 0.10 & 0.10 & 0.3821 \\
Humanoid-Walk & 0.9  & 0.520 & 0.10 & 0.10 & 0.3745 \\
Humanoid-Walk & 0.99 & 0.520 & 0.10 & 0.10 & 0.3732 \\
\midrule
MW-Push & 0.1 & 0.864 & 0.40 & 0.06 & 0.3810 \\
MW-Push & 0.5 & 0.550 & 0.08 & 0.06 & 0.3465 \\
MW-Push & 0.9 & 0.509 & 0.08 & 0.06 & 0.3334 \\
\bottomrule
\end{tabular}
\end{table}

\begin{figure}[t]
    \centering
    \begin{subfigure}[t]{0.31\linewidth}
        \centering
        \includegraphics[width=\linewidth]{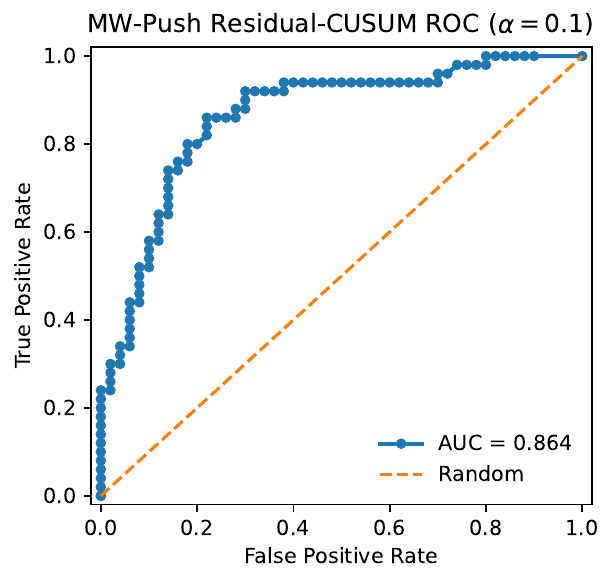}
        \subcaption{$\alpha=0.1$}
    \end{subfigure}
    \hfill
    \begin{subfigure}[t]{0.31\linewidth}
        \centering
        \includegraphics[width=\linewidth]{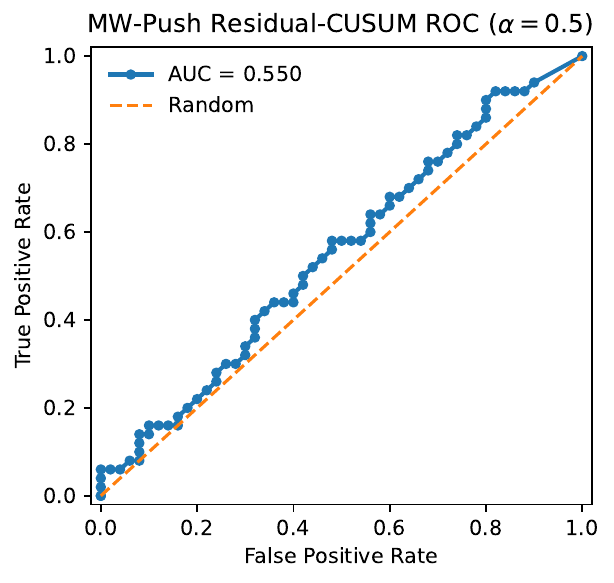}
        \subcaption{$\alpha=0.5$}
    \end{subfigure}
    \hfill
    \begin{subfigure}[t]{0.31\linewidth}
        \centering
        \includegraphics[width=\linewidth]{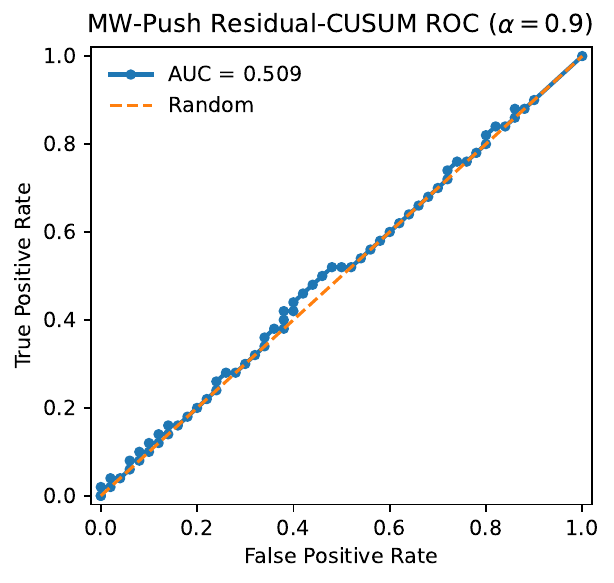}
        \subcaption{$\alpha=0.9$}
    \end{subfigure}

    \caption{
    Trajectory-level ROC curves for the residual-based CUSUM detector on MW-Push. Each candidate fine-tuning trajectory is assigned the score $\max_t S_t$, and ROC curves are computed by treating clean and poisoned trajectories as negative and positive examples. The aggressive low-stealth setting $\alpha=0.1$ is more detectable, while the stealth-prioritized setting $\alpha=0.9$ lies close to the chance diagonal, consistent with Table~\ref{tab:cusum_residual}.
    }
    \label{fig:cusum_roc}
\end{figure}

Table~\ref{tab:cusum_residual} and Figure~\ref{fig:cusum_roc} show that the detector separates aggressive low-$\alpha$ attacks more clearly, especially on MW-Push: at $\alpha=0.1$, AUC reaches $0.864$ and TPR is $0.40$ at empirical FPR $0.06$. However, under the stealth-prioritized settings used in the main experiments, the detector has little power: MW-Push at $\alpha=0.9$ gives AUC $0.509$ and TPR $0.08$ at empirical FPR $0.06$, while Humanoid-Walk at $\alpha=0.9$ gives AUC $0.520$ and TPR $0.10$ at empirical FPR $0.10$. These results do not prove indistinguishability against all sequential detectors; rather, they show that a natural residual-based CUSUM screen detects aggressive low-stealth poisoning but loses power under SWAAP's stealth-regularized operating points.

\subsubsection{Latent-Space PCA Visualization}
\label{app:pca}

We additionally visualize data-level stealthiness using PCA in the latent space used by the world model. This visualization is intended as a qualitative sanity check complementary to the scalar deviations in Table~\ref{main_results}. For Humanoid-Walk with $\alpha=0.9$, we take the $500$ transitions selected for poisoning and compare their clean next-state latents $z'$, poisoned next-state latents $\tilde z'$, and clean/reference-model predicted next-state latents $P_{\psi_{\mathrm{ref}}}(z,a)$. The poisoned-target plot illustrates the deviation from true next-states, while the clean-model-prediction plot illustrates the natural reference residual scale corresponding to $\delta_{m,\mathrm{ref}}$.

To make the visualization comparable, PCA is fitted on the pooled latent samples used in each panel and the first two principal components are plotted. Figure~\ref{fig:pca_stealth} shows that poisoned next-state targets largely overlap with their clean next-state counterparts and that their centroid remains close to the clean centroid. The clean/reference-model predictions show a similar overlapping pattern relative to true next states. This supports the interpretation that SWAAP's poisoned targets are not obvious latent outliers under this two-dimensional projection. Since PCA only captures the leading low-dimensional variation, this result should not be interpreted as proof of indistinguishability; rather, it provides qualitative evidence consistent with the $\delta_d$ and $\delta_{\mathrm{ref}}$ measurements.

\begin{figure}
    \centering
    \includegraphics[scale=0.33]{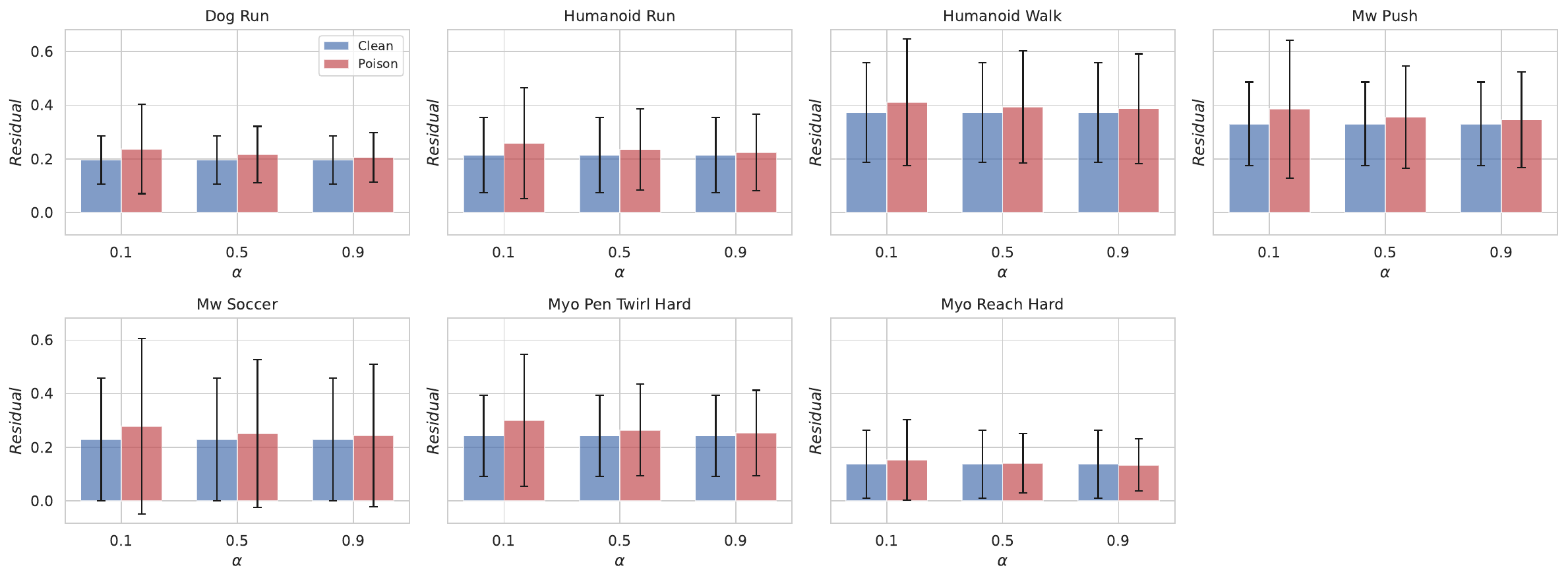}
    \caption{
    Pre-training data-level deviation across stealth weights $\alpha$. Blue bars denote the clean reference residual $\delta_{\mathrm{d,ref}}$ and red bars denote poisoned-data deviation $\delta_d$. Across $\alpha$, poisoned targets remain close to the clean reference residual scale in most environments, even when lower $\alpha$ increases attack strength.
    }
    \label{pretraining_detection}
\end{figure}

\begin{figure}[t]
    \centering
    \begin{subfigure}[t]{0.48\linewidth}
        \centering
        \includegraphics[width=\linewidth]{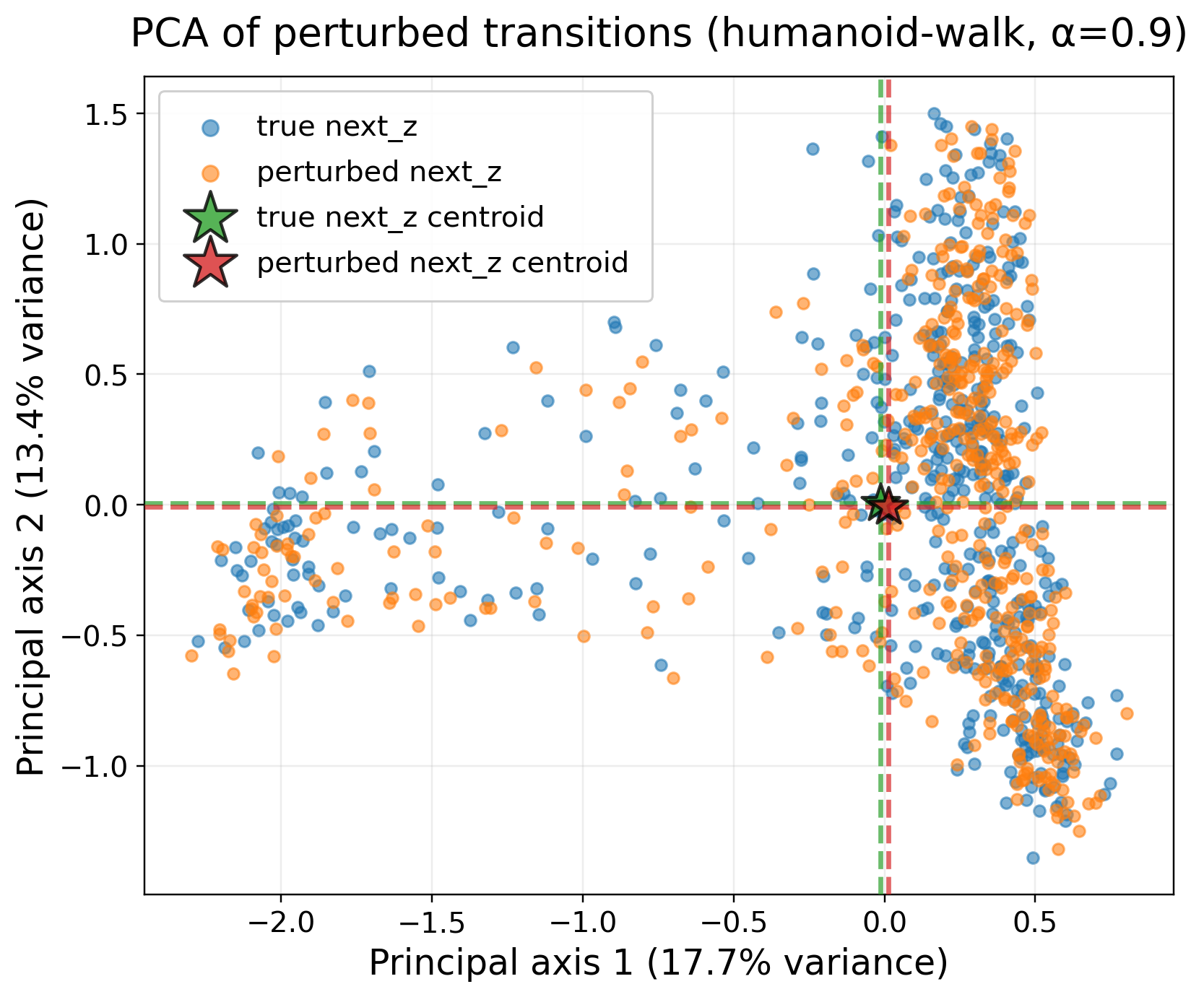}
        \subcaption{Poisoned targets vs. clean next states}
    \end{subfigure}
    \hfill
    \begin{subfigure}[t]{0.48\linewidth}
        \centering
        \includegraphics[width=\linewidth]{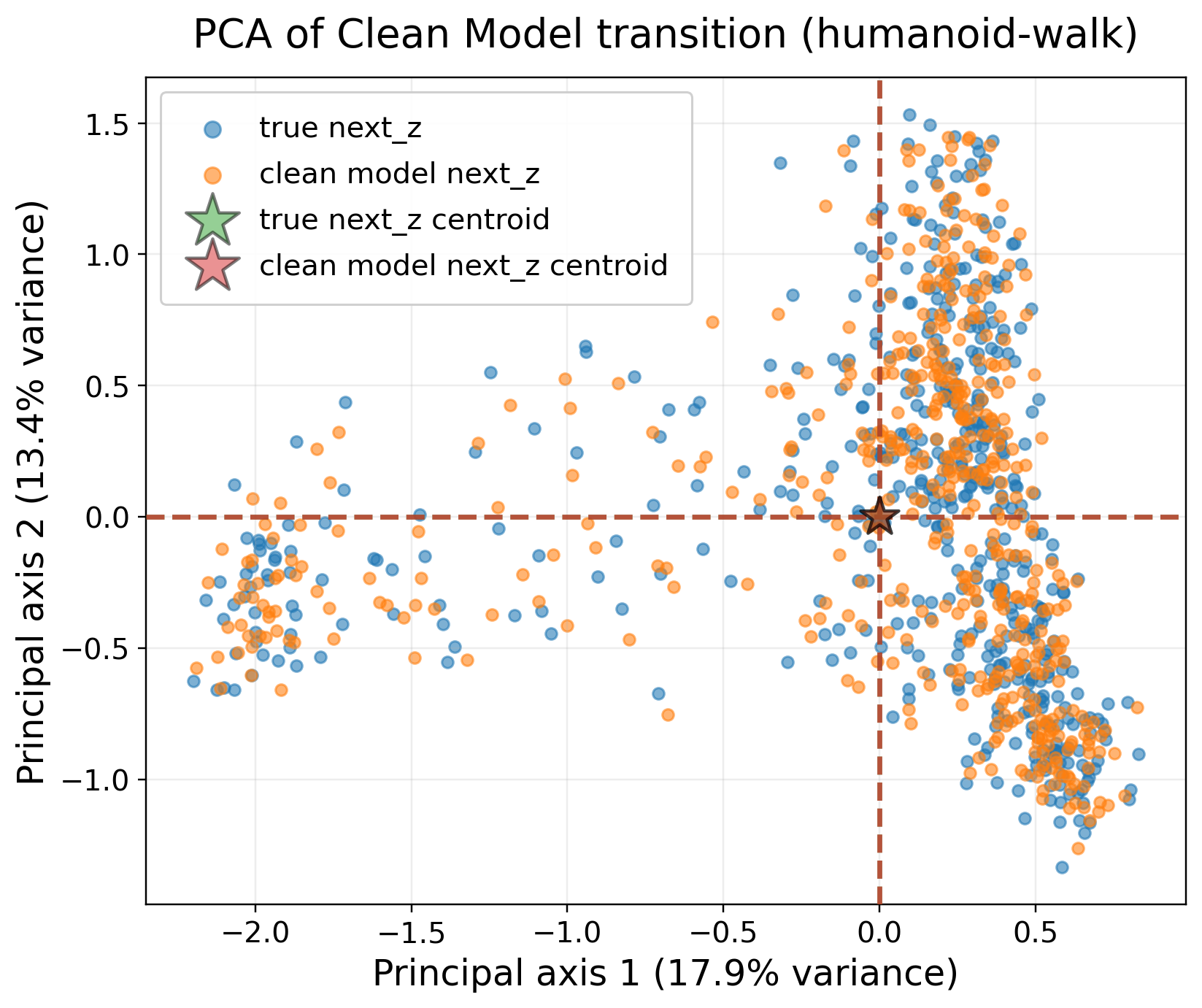}
        \subcaption{Reference-model predictions vs. clean next states}
    \end{subfigure}
    \caption{
    Latent-space PCA visualization on Humanoid-Walk. 
    (a) compares poisoned next-state targets $\tilde z'$ with their clean sources $z'$, illustrating the data-level perturbation relative true states, note this is different from $\delta_d$. 
    (b) compares clean/reference-model predictions $P_{\psi_{\mathrm{ref}}}(z,a)$ with true next states $z'$, illustrating the natural reference residual scale $\delta_{\mathrm{ref}}$. 
    Poisoned targets largely overlap with clean next states, similar to the clean model's natural prediction residuals.
    }
    \label{fig:pca_stealth}
\end{figure}

\subsection{Stage 1--Stage 2 Realization Bottleneck}
\label{app:stage_bottleneck}

SWAAP separates target identification from target realization. Stage~1 searches for a target world model $P_{\hat{\psi}}$ that induces low return while remaining close to clean dynamics. Direct Model Poisoning evaluates this target model directly and therefore measures the strength of Stage~1 under an unrealistic model-overwrite threat model. Stage~2 then attempts to realize the selected target model using only bounded data poisoning. Since the attacker can modify only an $r_p$ fraction of fine-tuning transitions and must keep poisoned targets stealthy, Stage~2 is expected to realize only part of the effect achieved by direct model overwrite.

We quantify this realization gap on Humanoid-Walk. With $\alpha=0.1$ and $r_p=0.3$, the clean model has deviation $\delta_m(P_{\hat{\psi}},P_{\psi_0})=0.15$ from the Stage~1 target, while the Stage~2 poisoned model reaches $\delta_m(P_{\hat{\psi}},P_{\psi})=0.11$ relative to the target model. This shows that Stage~2 recovers a substantial portion of the Stage~1 target, but does not exactly match it. More broadly, Figure~\ref{fig:scatter_comparison} shows this bottleneck as a gap between the Direct Model Poisoning frontier and the SWAAP frontier: in Humanoid-Walk, SWAAP closely tracks the Direct Model Poisoning frontier, while in Myo-Pen-Twirl-Hard the gap is larger, indicating that some harmful target models are more difficult to realize through gradient-matched poisoned data.

We further test whether directly using the Stage~1 target predictions as poisoned data is sufficient. We define \textbf{DMP-data} as a variant that replaces the selected top-$r_p$ next states with predictions from the Stage~1 target model $P_{\hat{\psi}}$, without optimizing the Stage~2 gradient-matching objective. DMP-data remains stealthy at the data level, with $\delta_d=0.0872\pm0.0425$, comparable to the clean reference residual $0.0897\pm0.0455$. However, it is less effective than SWAAP, especially under robust training. Without TRIM, DMP-data gives $R=727\pm205$ and $\delta_m=0.116\pm0.065$, while SWAAP gives $R=521\pm332$ and $\delta_m=0.145\pm0.079$. Under TRIM, DMP-data gives $R=599\pm297$ and $\delta_m=0.130\pm0.071$, while SWAAP gives $R=183\pm356$ and $\delta_m=0.192\pm0.075$. Thus, directly copying target-model next states is not enough; the gradient-matching realization step better steers fine-tuning toward the target dynamics.

Overall, these results identify target realization under attack--stealth constraints as the main practical bottleneck and an opportunity to improve the current two-stage implementation. Stage~1 can identify harmful low-deviation target models, but Stage~2 must approximate them using a small number of poisoned transitions. This explains why Direct Model Poisoning is often stronger than SWAAP, and also why improving Stage~2 or making Stage~1 aware of Stage~2 realizability is a promising direction for future work.

\section{Ablations and Baselines}\label{app:ablations}

\subsection{Iterative FGSM/PGD Transition-Perturbation Baselines}
\label{app:pgd_baseline}

We compare SWAAP with iterative signed-gradient transition-perturbation baselines adapted from adversarial state-perturbation attacks. These baselines do not identify a target world model or match fine-tuning gradients. Instead, they directly perturb next-state targets to lower the value of a prior policy $\pi_\theta$ trained from the clean world model.

For a clean transition $(s,a,s')$, we initialize $\tilde{s}'_0=s'$ and run $K$ projected signed-gradient steps:
\[
\tilde{s}'_{k+1}
=
\Pi_{[s'-\epsilon,s'+\epsilon]}
\left(
\tilde{s}'_k
-
\eta\,\operatorname{sign}
\left(
\nabla_{\tilde{s}'_k} Q(\tilde{s}'_k,\pi_\theta(\tilde{s}'_k))
\right)
\right),
\qquad k=0,\ldots,K-1,
\]
where $\eta$ is the step size, $\epsilon$ bounds the maximum per-coordinate perturbation, and $\Pi_{[s'-\epsilon,s'+\epsilon]}$ denotes projection to the $\ell_\infty$ box around the original next state. When $K=1$, this reduces to an FGSM-style perturbation; for $K>1$, it is an iterative-FGSM / $\ell_\infty$ PGD-style baseline.

We evaluate two variants. The \textbf{offline transition-perturbation baseline} applies this replacement to fine-tuning transitions before model update, matching our data-poisoning setting. The \textbf{online state-perturbation baseline} applies an analogous perturbation at test time to the state observed by the victim. The online variant is not the same threat model as SWAAP because it requires test-time intervention, but it provides context against a standard adversarial-state attack.

\begin{table}[t]
\centering
\caption{
Offline iterative-FGSM/$\ell_\infty$-PGD transition-perturbation baseline on MW-Push. The baseline directly perturbs next-state targets to reduce the value of a clean prior policy, without Stage~1 target-model identification or Stage~2 gradient matching. SWAAP achieves a lower return while keeping $\delta_d$ comparable or lower. Values are mean $\pm$ std.
}
\label{tab:pgd_offline}
\begin{tabular}{lccc}
\toprule
Method & Return & $\delta_m$ & $\delta_d$ \\
\midrule
Clean & $1784 \pm 13$ & $0.080 \pm 0.002$ & $0.0799$ \\
PGD ($\epsilon=0.015$) & $1732 \pm 56$ & $0.107 \pm 0.012$ & $0.0832$ \\
PGD ($\epsilon=0.05$) & $1666 \pm 76$ & $0.121 \pm 0.008$ & $0.0983$ \\
SWAAP & $1641 \pm 80$ & $0.171 \pm 0.011$ & $0.0805$ \\
\bottomrule
\end{tabular}
\end{table}

Table~\ref{tab:pgd_offline} shows that the offline signed-gradient transition-perturbation baseline is weaker than SWAAP on MW-Push. Increasing $\epsilon$ improves its attack strength but also increases $\delta_d$, making the poisoned data less stealthy before fine-tuning. In contrast, SWAAP achieves a lower return with $\delta_d$ close to the clean reference scale, because it poisons data to steer the fine-tuning gradient toward a harmful target world model rather than applying only local value-reducing perturbations.

For the online state-perturbation comparison, iterative signed-gradient perturbation with $\epsilon=0.01$ gives $R=1785\pm5$ and $\delta_m=0.106\pm0.003$, while $\epsilon=0.02$ gives $R=1719\pm48$ and $\delta_m=0.159\pm0.010$. SWAAP gives $R=1641\pm80$ and $\delta_m=0.171\pm0.011$. This comparison is not a direct threat-model match: online perturbation requires modifying states at test time, whereas SWAAP modifies only fine-tuning data and induces later failures through the learned world model. Nevertheless, it shows that SWAAP can produce comparable or stronger degradation without test-time intervention.

\subsection{Top-$r_p$ Transition Selection vs. Random Selection}
\label{app:top_rp}

Stage~2 first selects the subset of transitions to poison. Under a fixed poisoning budget $r_p$, our default heuristic selects the top-$r_p$ transitions with the largest deviation from the Stage~1 target model. These transitions are the most inconsistent with the target dynamics, so modifying them should provide a stronger fine-tuning signal toward the target than modifying randomly selected transitions.

We test this heuristic on Humanoid-Walk using the same Stage~1 target model and the same poisoning budget. Replacing top-$r_p$ selection with random transition selection weakens the attack: top-$r_p$ selection gives $R=775\pm50$ and $\delta_m=0.110\pm0.007$, while random selection gives $R=822\pm13$ and $\delta_m=0.102\pm0.002$. Thus, selecting transitions most inconsistent with the Stage~1 target improves attack effectiveness at comparable model-level deviation.

\subsection{MPC Hyperparameter Ablation}
Table~\ref{mpc_compare} reports the reward of TD-MPC2 under different rollout horizons ($h=3,6,9$) and number of samples for both the clean and SWAAP settings. We observe that our SWAAP attack consistently reduces the agent’s performance across all configurations, demonstrating its effectiveness under every setting. Additionally, the MPC performance generally increases with the number of samples and decreases as the rollout horizon grows. Consequently, the impact of our attack is also influenced by these hyperparameters: it tends to be more significant when the baseline MPC performance is higher (larger sample sizes) and slightly less effective at longer horizons given the low clean reward.

\begin{table}[!t]
\centering
\caption{Effect of MPC hyperparameters ($\texttt{num\_samples}$, horizon $H$) on reward under clean and SWAAP settings evaluated in Humanoid-Walk. Values are mean $\pm$ std over 10 episodes.}
\label{mpc_compare}
\begin{tabular}{|c|c|ccc|}
\hline
\multirow{2}{*}{Num Samples} & \multirow{2}{*}{} & \multicolumn{3}{c|}{Return}                                                             \\ \cline{3-5} 
                             &                   & \multicolumn{1}{c|}{$h=3$}         & \multicolumn{1}{c|}{$h=6$}         & $h=9$         \\ \hline
\multirow{2}{*}{64}          & Clean             & \multicolumn{1}{c|}{12 $\pm$ 21}   & \multicolumn{1}{c|}{5 $\pm$ 4}     & 9 $\pm$ 21    \\
                             & SWAAP             & \multicolumn{1}{c|}{7 $\pm$ 7}     & \multicolumn{1}{c|}{19 $\pm$ 39}   & 14 $\pm$ 20   \\ \hline
\multirow{2}{*}{128}         & Clean             & \multicolumn{1}{c|}{150 $\pm$ 95}  & \multicolumn{1}{c|}{58 $\pm$ 92}   & 2 $\pm$ 2     \\
                             & SWAAP             & \multicolumn{1}{c|}{101 $\pm$ 85}  & \multicolumn{1}{c|}{38 $\pm$ 51}   & 43 $\pm$ 41   \\ \hline
\multirow{2}{*}{256}         & Clean             & \multicolumn{1}{c|}{788 $\pm$ 157} & \multicolumn{1}{c|}{488 $\pm$ 279} & 68 $\pm$ 115  \\
                             & SWAAP             & \multicolumn{1}{c|}{347 $\pm$ 257} & \multicolumn{1}{c|}{93 $\pm$ 93}   & 21 $\pm$ 36   \\ \hline
\multirow{2}{*}{512}         & Clean             & \multicolumn{1}{c|}{855 $\pm$ 62}  & \multicolumn{1}{c|}{687 $\pm$ 112} & 37 $\pm$ 61   \\
                             & SWAAP             & \multicolumn{1}{c|}{594 $\pm$ 203} & \multicolumn{1}{c|}{49 $\pm$ 61}   & 41 $\pm$ 58   \\ \hline
\multirow{2}{*}{1024}        & Clean             & \multicolumn{1}{c|}{875 $\pm$ 63}  & \multicolumn{1}{c|}{812 $\pm$ 91}  & 385 $\pm$ 287 \\
                             & SWAAP             & \multicolumn{1}{c|}{656 $\pm$ 192} & \multicolumn{1}{c|}{87 $\pm$ 147}  & 75 $\pm$ 92   \\ \hline
\end{tabular}
\end{table}

\subsection{Stage 1 Masked Perturbation Fraction}\label{app:stage1_fraction}
\begin{table}[!t]
    \caption{Attack performance with different poison fraction in Stage 1 in Humanoid-Walk with $\alpha$ = 0.9}
    \label{tab:perturb_rate_stage1}
    \centering
\begin{tabular}{|c|c|c|}
\hline
Stage 1 poison fraction & Return    & $\delta_m$    \\ \hline
0\% (clean)             & 866 ± 57  & 0.09 ± 0.05 \\ \hline
10\%                    & 825 ± 106 & 0.10 ± 0.06 \\ \hline
50\%                    & 669 ± 242 & 0.12 ± 0.07 \\ \hline
100\%                   & 521 ± 332 & 0.14 ± 0.08 \\ \hline
\end{tabular}
\end{table}

Stage 1 of SWAAP can be adjusted to only perturb a certain fraction of the state-action pairs' transitions. We have conducted additional experiments on perturbing the top 10\% and 50\% most influential state-action pairs compared to perturbing all pairs considered before 
in Humanoid-Walk with $r_p = 0.1, \alpha = 0.9$ in Table~\ref{tab:perturb_rate_stage1}. Stage~1 optimizes the parameters of a global neural world model, so changing $\hat{\psi}$ can in principle affect predictions beyond any selected subset of inputs. In this ablation, “perturbing a fraction of state-action pairs” therefore means that we restrict the Stage~1 perturbation objective to a masked subset of rollout transitions, rather than enforcing a hard functional constraint that the model output is unchanged elsewhere. Specifically, after collecting rollout transitions, we compute an importance score for each state using the value-spread criterion from~\citet{liang2022efficient},
\[
w(s)=\max_{a_1\in\mathcal A}Q^\pi(s,a_1)-\min_{a_2\in\mathcal A}Q^\pi(s,a_2).
\]
We then select the top-$w_c$ fraction of transitions by this score and apply the Stage~1 target-model perturbation loss only on this selected subset. The rest of the transitions still contribute to the usual model-level deviation/stealth regularization when applicable, but they do not receive direct target-perturbation pressure. Thus, the ablation measures how the attack changes when Stage~1 is guided by increasingly sparse sets of high-importance transitions, not an exact guarantee that the learned neural model changes only on those inputs.

Table~\ref{tab:perturb_rate_stage1} shows that increasing this selected fraction strengthens the attack but also increases $\delta_m$, giving a tunable effectiveness--stealth trade-off.

\subsection{Stage 2 Poisoning Ratio $r_p$}\label{app:rp_ablation}
\begin{table}[!t]
    \caption{Attack performance with different data poison ratio ($r_p$) in Stage 2 in Humanoid-Walk with $\alpha$ = 0.9}
    \label{tab:r_p}
    \centering
    \begin{tabular}{|c|c|c|}
    \hline
    $r_p$ & Return    & $\delta_m$    \\ \hline
    clean & 866 ± 57  & 0.09 ± 0.05 \\ \hline
    0.01  & 844 ± 94  & 0.10 ± 0.05 \\ \hline
    0.05  & 739 ± 172 & 0.11 ± 0.07 \\ \hline
    0.10  & 521 ± 332 & 0.14 ± 0.08 \\ \hline
    0.20  & 317 ± 339 & 0.18 ± 0.08 \\ \hline
    0.30  & 242 ± 305 & 0.19 ± 0.08 \\ \hline
    \end{tabular}
\end{table}
We conduct an ablation study on different $r_p$ ratios in Stage 2, with results in Table~\ref{tab:r_p}. We observe that with lower poisoning rates (1\% and 5\%), the attack performance decreases as expected, but it still consistently induces noticeable degradation in return.

\subsection{Trajectory-Consistent Data Poisoning}
\label{app:trajectory_dp}

The Stage~2 formulation in Eq.~\ref{dp_objective} treats the fine-tuning data as a transition buffer, where each tuple $(s_t,a_t,s_{t+1})$ can be poisoned by replacing only the next-state target. In sequential trajectory data, however, the same intermediate state appears in two adjacent tuples:
\[
(s_t,a_t,s_{t+1}), \qquad (s_{t+1},a_{t+1},s_{t+2}).
\]
Therefore, replacing $s_{t+1}$ only as the next-state target of the first tuple would make the stored trajectory internally inconsistent. To adapt SWAAP to this setting, we use a trajectory-consistent update: when transition $t$ is selected for poisoning, we replace both occurrences of the intermediate latent state,
\[
(s_t,a_t,s_{t+1}) \mapsto (s_t,a_t,\tilde{s}_{t+1}), \qquad
(s_{t+1},a_{t+1},s_{t+2}) \mapsto (\tilde{s}_{t+1},a_{t+1},s_{t+2}).
\]
Terminal transitions are excluded so that $t+1$ always belongs to the same trajectory. In the gradient-matching step, we add a SARSA-style consistency term so the matching gradient becomes
\[
G_{\text{real}}
=
\mathbb{E}_{(s_t,a_t, s_t',a_{t+1},s_{t+1}')\sim \tilde D}
\Big[
\nabla_{\psi_0} \Big(
\| \tilde{s}'_{t}-P_{\psi_0}(s_t,a_t)\|_2^2
+
\| s'_{t+1}-P_{\psi_0}(\tilde{s}'_{t},a_{t+1})\|_2^2
\Big)\Big],
\]
which encourages the poisoned intermediate state to be both reachable from $(s_t,a_t)$ and predictive of the observed successor under $a_{t+1}$. After each poisoned-target update, we explicitly enforce $s_{t+1}=\tilde{s}_{t+1}$ in the next tuple, so the saved poisoned data remain trajectory-consistent.

On Humanoid-Walk, trajectory-consistent poisoning with $r_p=0.05$ gives return $784\pm19$ and model-level deviation $\delta_m=.111\pm.003$, with data-level reference deviation $\delta_d=.102\pm.007$. Compared with the standard transition-buffer setting in Table~\ref{main_results}, the attack is weaker, as expected, because the consistency constraint reduces the degrees of freedom available to the poisoner. Nevertheless, the result shows that SWAAP can be adapted from independent transition buffers to sequential fine-tuning trajectories.

\subsection{Larger Fine-Tuning Buffers}\label{app:larger_buffer}
We conducted additional experiments on \texttt{HumanoidWalk} using fine-tuning dataset sizes ranging from 5{,}000 to 25{,}000, with a fixed budget of $500$ poisoned transitions. The results are shown in Table~\ref{tab:ft_sizes}. Although increasing the fine-tuning dataset size reduces the magnitude of performance degradation, as larger clean buffers naturally dilute the poisoned portion, SWAAP still consistently induces meaningful return drops across all tested sizes, demonstrating its robustness even when the victim fine-tunes on substantially larger datasets.

\begin{table}[!t]
\centering
\caption{Effect of fine-tuning dataset size in HumanoidWalk.}
\label{tab:ft_sizes}
\begin{tabular}{lcc}
\toprule
\textbf{Fine-tuning Dataset Size} & \textbf{Return} & \textbf{$\delta_m$} \\
\midrule
5{,}000  & $521 \pm 332$ & $0.14 \pm 0.08$ \\
10{,}000 & $540 \pm 348$ & $0.14 \pm 0.08$ \\
15{,}000 & $550 \pm 306$ & $0.14 \pm 0.08$ \\
20{,}000 & $607 \pm 300$ & $0.13 \pm 0.08$ \\
25{,}000 & $679 \pm 234$ & $0.12 \pm 0.07$ \\
\bottomrule
\end{tabular}
\end{table}

\section{Generalization Experiments}\label{app:generalization}

\subsection{Additional MyoSuite and MetaWorld Environments}
We conducted additional experiments on five MyoSuite and MetaWorld tasks, with results reported in Table~\ref{tab:extra_envs}. All experiments use $r_p = 0.1$ and are evaluated over 100 episodes. For \texttt{Obj-Hold-Hard}, \texttt{Key-Turn}, and \texttt{Coffee-Pull}, we set $\alpha = 0.9$, while for \texttt{Pose} and \texttt{Door-Open}, we use $\alpha = 0.1$. These results further confirm that SWAAP reliably reduces the return while keeping the poisoned world model close to the clean model, consistent with our main findings.

\begin{table}[!t]
\centering
\caption{Additional evaluation on MyoSuite and MetaWorld environments. Each test is a single run with $100$ episodes.}
\label{tab:extra_envs}
\resizebox{0.6\textwidth}{!}{
\begin{tabular}{lcccc}
\toprule
\textbf{Environment} & \textbf{Natural Return} & \textbf{SWAAP Return} & \textbf{Natural $\delta_m$} & \textbf{SWAAP $\delta_m$} \\
\midrule
myo-obj-hold-hard & $-289 \pm 2343$  & $-2671 \pm 4241$ & $0.10 \pm 0.07$ & $0.13 \pm 0.05$ \\
myo-pose           & $696 \pm 4$      & $665 \pm 157$    & $0.04 \pm 0.04$ & $0.08 \pm 0.06$ \\
myo-key-turn       & $1125 \pm 196$   & $940 \pm 375$    & $0.19 \pm 0.14$ & $0.11 \pm 0.08$ \\
mw-coffee-pull   & $1511 \pm 25$    & $1402 \pm 318$   & $0.08 \pm 0.04$ & $0.09 \pm 0.06$ \\
mw-door-open     & $1551 \pm 58$    & $1310 \pm 445$   & $0.04 \pm 0.02$    & $0.06 \pm 0.05$ \\
\bottomrule
\end{tabular}
}
\end{table}

\subsection{DINO-WM Push-T Results}\label{app:dino}
We adapted our method to the DINO-WM world models for goal-conditioned tasks~\citep{zhou2025dinowmworldmodelspretrained}. Our experimental setup largely follows the configuration used in the original implementation; the detailed settings are provided in Table~\ref{tab:dino_setup}.

For Stage 1, we used a goal-conditioned RL agent as the surrogate policy. Preliminary results show that the identified target world model reduces the success rate on the Push-T task from $92\%$ to $72\%$, with the relative deviation increasing from $\delta_m = 0.11 \pm 0.06$ to $\delta_m = 0.24 \pm 0.10$.

We then conducted additional experiments incorporating Stage 2 of our attack pipeline. Using the perturbed model identified in Stage 1, we employed gradient matching to generate poisoned samples (poisoning ratio 0.15) and fine-tuned a DINO-WM model on the resulting poisoned dataset. On the Push-T task, performance dropped from $92\%$ to $77\%$, with $\delta_m = 0.16 \pm 0.09$. These early results suggest that the method can extend beyond TD-MPC2 to goal-conditioned visual world models, although a broader evaluation is left for future work.

\begin{table}[!t]
\centering
\small
\caption{Key hyperparameters used in the DINO-WM Push-T experiment}
\label{tab:dino_setup}
\begin{tabularx}{\textwidth}{l l X}
\toprule
\textbf{Hyperparameter} & \textbf{Typical value(s)} & \textbf{Description} \\
\midrule

obs\_shape &
(3, 224, 224) &
Shape of pixel observation $s$. \\

latent\_shape &
(196, 394) &
Shape of latent state $z = \mathrm{enc}(s)$. \\

action\_dim &
10 &
Dimension of actions. \\

num\_hist &
3 &
Number of historical latent states used to predict the next state. \\

H &
5 &
Planning horizon length (number of rollout steps during planning), which also corresponds to the number of planned actions. \\

num\_samples &
100 &
Total number of sampled trajectories during planning. \\

opt\_steps &
30 &
Number of optimization iterations per planning step using cross-entropy method (CEM). \\

\bottomrule
\end{tabularx}
\end{table}

\subsection{Gray-Box Attack on TD-MPC2 Multi-Task World Model}\label{app:graybox}
We conducted a preliminary gray-box experiment in which the attacker only has access to a surrogate world model trained on a single task, while the victim employs a multi-task world model. The attacker performs Stage 1 and Stage 2 entirely on the surrogate model and then supplies the resulting poisoned dataset to the victim for fine-tuning.
In this experiment, the surrogate is a TD-MPC2 world model trained on MW-Soccer, whereas the victim is a multi-task world model used across multiple tasks. We use $r_p = 0.1, \alpha = 0.1$ in this experiment. Our initial results show that SWAAP still induces substantial performance degradation across multiple downstream tasks (see Table~\ref{tab:gray-box}), despite the mismatch between the surrogate and victim models.
\begin{table}[!t]
\centering
\small
\caption{Gray-box attack on TD-MPC2 multi-task world model}
\label{tab:gray-box}
\begin{tabular}{l l l c c}
\toprule
\textbf{Source Task} & \textbf{Target Task} & \textbf{Setting} 
& \textbf{Return (R $\pm$ std)} 
& $\boldsymbol{\delta_m}$ \textbf{($\pm$ std)} \\
\midrule
MW-Soccer & MW-Soccer & Natural & $1380 \pm 382$ & $0.11 \pm 0.11$ \\
MW-Soccer & MW-Soccer & SWAAP   & $1115 \pm 442$ & $0.13 \pm 0.12$ \\
MW-Soccer & MW-Push   & Natural & $1290 \pm 609$ & $0.13 \pm 0.16$ \\
MW-Soccer & MW-Push   & SWAAP   & $1064 \pm 692$ & $0.16 \pm 0.19$ \\
\bottomrule
\end{tabular}
\end{table}

\end{document}

%% file: math_commands.tex
\usepackage{amsmath,amsfonts,bm}

\def\eqref#1{equation~\ref{#1}}
\def\Eqref#1{Equation~\ref{#1}}

\def\1{\bm{1}}

\DeclareMathAlphabet{\mathsfit}{\encodingdefault}{\sfdefault}{m}{sl}
\SetMathAlphabet{\mathsfit}{bold}{\encodingdefault}{\sfdefault}{bx}{n}

\DeclareMathOperator*{\argmax}{arg\,max}
\DeclareMathOperator*{\argmin}{arg\,min}